%% file: main.tex
\documentclass[11pt]{article}

\usepackage[preprint]{acl}

\usepackage{times}
\usepackage{latexsym}
\usepackage[T1]{fontenc}
\usepackage[utf8]{inputenc}
\usepackage{microtype}
\usepackage{inconsolata}
\usepackage{graphicx}
\usepackage{booktabs}
\usepackage{amsmath,amssymb,amsfonts,amsthm}
\usepackage{bm}
\usepackage{multirow}
\usepackage{xcolor}
\usepackage{colortbl}
\usepackage{array}
\usepackage{tabularx}
\usepackage{enumitem}
\usepackage[ruled,vlined]{algorithm2e}
\SetKwInput{KwInput}{Input}
\SetKwInput{KwOutput}{Output}

\newtheorem{proposition}{Proposition}

\newtheorem{corollary}{Corollary}

\newtheorem{remark}{Remark}

\definecolor{oursbg}{RGB}{253,238,240}
\definecolor{simred}{RGB}{165,30,45}

\newcommand{\method}{\textsc{Arcus}}
\newcommand{\E}{\mathbb{E}}

\newcommand{\Bin}{\mathrm{Bin}}
\newcommand{\N}{\mathcal{N}}
\newcommand{\pst}{\psi^{\star}}
\newcommand{\hpi}{\tfrac{\pi}{2}}
\newcommand{\yhat}{\hat{y}}
\newcommand{\Yhat}{\widehat{Y}}

\input{tables/stats.tex}

\title{Prompts Live on an Arc: Gaussian Curricula in Fisher--Rao\\Coordinates for Rollout-Efficient GRPO}

\author{
  \textbf{Mei Okonkwo\textsuperscript{1}},
  \textbf{Pixel Nomand\textsuperscript{1}},
  \textbf{Julian Berg\textsuperscript{2}},
  \textbf{Elena Voss\textsuperscript{1}},
  \\
  \textbf{Lena Park\textsuperscript{1}},
  \textbf{Marcus Hale\textsuperscript{2}},
  \textbf{Adrian Cho\textsuperscript{2}},
  \textbf{Sofia Reyes\textsuperscript{1}}
  \\
  \\
  \textsuperscript{1}University of Wisconsin--Madison \\
  \textsuperscript{2}University of Washington
}

\begin{document}
\maketitle

\begin{abstract}
Group relative policy optimization (GRPO) learns only from prompts whose sampled responses disagree: a group that is entirely correct or entirely incorrect has zero reward variance, contributes no gradient, and still consumes its rollouts. Prompt-selection methods reduce this waste by steering sampling toward intermediate pass rates, but they choose the target, its width, and the uncertainty model heuristically, in raw pass-rate or logit coordinates. We show that GRPO comes with a natural coordinate for pass rates: the arc length $\psi=\arcsin\sqrt{p}$ on the Bernoulli Fisher--Rao manifold. In arc length, the expected GRPO update is uniform up to two boundary ramps; the probability of a zero-variance group is bounded by two Gaussian boundary layers of width $1/\sqrt{2G}$; pass-rate evidence has constant noise; and the gradients of the pass@$k$ and pass$^k$ objectives are Gaussians whose center and width follow from $k$ in closed form. A prompt curriculum for GRPO is therefore a Gaussian in arc length, and choosing its center amounts to choosing the objective. We turn this observation into \method{}, a drop-in sampler that tracks every prompt with a Kalman filter in arc length, scores prompts by an objective-matched Gaussian kernel times the predicted probability of an informative group, keeps only informative groups for the unchanged GRPO update, and paces the target toward the hardest objective whose predicted yield stays within a small slack of the best. Across six mathematical reasoning benchmarks and three backbones, \method{} improves the average accuracy of GRPO by \gainGRPOmin--\gainGRPOmax{} points and that of dynamic sampling by \gainDSmin--\gainDSmax{} points, while generating \saveDSmin--\saveDSmax\% fewer rollouts than dynamic sampling.
\end{abstract}

\section{Introduction}
\label{sec:intro}

Reinforcement learning with verifiable rewards (RLVR) is now the standard recipe for eliciting multi-step reasoning from language models~\citep{shao2024deepseekmath,deepseekai2025r1,kimi2025k15,lambert2024tulu}. Its workhorse, group relative policy optimization (GRPO), samples $G$ responses per prompt and normalizes each reward by the mean and standard deviation of its group~\citep{shao2024deepseekmath}. This removes the critic, but a prompt then teaches nothing unless its responses disagree: when all responses are correct or all are wrong, every advantage vanishes, yet the $G$ rollouts, the dominant cost of a step, have been paid for. Such \emph{zero-variance} groups make up about half of a uniformly sampled batch, and their share grows as the policy improves~\citep{yu2025dapo,zheng2025greso,le2026noprompt}.

Two families of methods attack the waste. \emph{Evaluate-then-filter} methods such as dynamic sampling (DS) oversample prompts and keep only informative groups~\citep{yu2025dapo,bae2026online}, which can double or triple the rollouts per step. \emph{Predict-then-select} methods forecast each prompt's pass rate from its history and sample near an intermediate target, using Beta posteriors~\citep{qu2026mopps}, hidden Markov dynamics~\citep{mao2026dps}, Kalman filters~\citep{zhu2026kgps}, skipping rules~\citep{zheng2025greso}, or Gaussian and softmax weights~\citep{lin2026fgexpo,zeng2025cures}; curricula also move the target by reward feedback~\citep{shi2026adarft,chen2025sec}. Three choices remain heuristic and are made independently: the \emph{coordinate} for comparing pass rates (raw $p$ or logit), the \emph{width} of the preferred region, and its \emph{target}, usually fixed at $p{=}0.5$.

These choices are not free, because GRPO has already fixed a geometry: its normalization makes the population update ascend $\sum_q 2\arcsin\sqrt{p_q}$~\citep{davis2025objective,mroueh2025grpo}. We observe that $\psi=\arcsin\sqrt{p}$ is the Fisher--Rao arc length of the Bernoulli family~\citep{rao1945information,amari2016information}, and that four seemingly unrelated quantities become simple on this arc (Figure~\ref{fig:geometry}). \emph{(i)} GRPO's expected update is uniform in $\psi$ up to two boundary ramps. \emph{(ii)} The probability of a zero-variance group is bounded by two Gaussian boundary layers, $e^{-G\psi^2}+e^{-G(\pi/2-\psi)^2}$. \emph{(iii)} Pass-rate evidence has constant noise, $1/(4n)$ per $n$ rollouts. \emph{(iv)} The gradients of pass@$k$ and pass$^k$ are Gaussians with center $\arcsin(1/\sqrt{2k})$ and width $1/(2\sqrt{k})$. Hence a prompt curriculum for GRPO should be a Gaussian in arc length whose width is set by the objective, whose uncertainty enters by adding variances, and whose center selects an objective along the family pass$^k\leftrightarrow$pass@$1\leftrightarrow$pass@$k$.

\begin{figure*}[t]
\centering
\includegraphics[width=\textwidth]{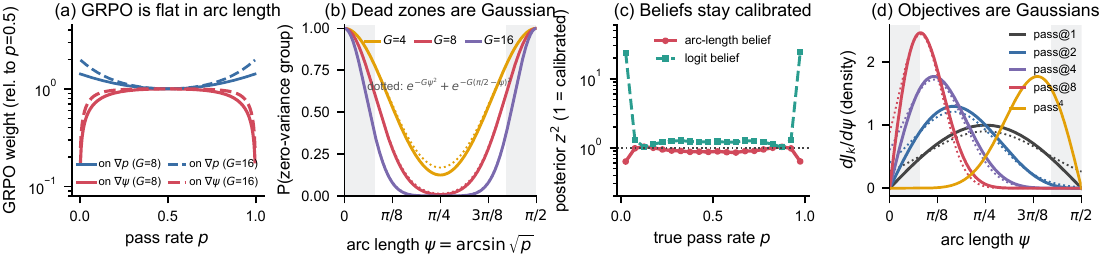}
\caption{\textbf{Prompts live on an arc} (exact curves). (a) GRPO's expected update as a weight on $\nabla p$ is U-shaped (blue); as a weight on $\nabla\psi$, $\psi=\arcsin\sqrt{p}$, it is flat up to two ramps (red; Prop.~\ref{prop:uniform}). (b) Zero-variance probability (solid) and its Gaussian bound (dotted; Prop.~\ref{prop:deadzone}). (c) After one group of $n{=}8$, arc-length posteriors are calibrated while logit posteriors are over-confident by $20\times$ where zero-variance prompts live (Prop.~\ref{prop:noise}). (d) Arc-length gradients of pass@$k$ and pass$^4$ (solid) and their Gaussian forms (dotted; Prop.~\ref{prop:kernel}).}
\label{fig:geometry}
\end{figure*}

We build \method{} (\emph{ARc-length CUrriculum Sampling}) on this geometry. It keeps a Kalman belief over each prompt's arc length. Mixed groups enter through a variance-stabilized observation and zero-variance groups through their boundary-layer likelihood, so the filter stays calibrated where logit filters fail, and a drift term removes the stale-easy bias of history-based selectors. Prompts are scored by their overlap with an objective-matched Gaussian kernel times the closed-form probability of an informative group. Only informative groups enter the unchanged GRPO update, and a pacing rule moves the target to the hardest objective whose predicted yield stays within $\varepsilon$ of the best. Every quantity is a closed-form scalar, and no rollout is spent on prediction. Our contributions:
\begin{itemize}[leftmargin=*,itemsep=1pt,topsep=2pt]
\item \textbf{Geometry.} We identify arc length as GRPO's native coordinate, prove that uniform updates, Gaussian dead zones, homoscedastic evidence, and Gaussian objective kernels hold in it, and show that every prompt sampler defines the objective GRPO ascends (\S\ref{sec:geometry}).
\item \textbf{Method.} \method{} derives its kernel width, uncertainty handling, dead-zone correction, and target pacing from the geometry and leaves GRPO's loss untouched (\S\ref{sec:method}).
\item \textbf{Evidence.} On three backbones, \method{} beats nine prompt-selection baselines, outperforms DS with \saveDSmin--\saveDSmax\% fewer rollouts, and reaches DS's final accuracy with \rollToDS\% fewer rollouts; ablations isolate the coordinate, kernel, and pacing (\S\ref{sec:experiments}).
\end{itemize}

\section{Related Work}
\label{sec:related}

\paragraph{Zero-variance prompts and prompt selection.}
DS filters zero-variance groups after generation~\citep{yu2025dapo} and online difficulty filtering keeps an accuracy band~\citep{bae2026online}; both pay for what they discard. GRESO skips recently uninformative prompts~\citep{zheng2025greso}. Rollout-free selectors use Beta posteriors with Thompson sampling~\citep{qu2026mopps}, hidden Markov dynamics~\citep{mao2026dps}, a learned cross-prompt predictor~\citep{qu2026gps}, or a logit-space Kalman filter~\citep{zhu2026kgps}; CurES samples by a softmax over $\sqrt{p(1-p)}$~\citep{zeng2025cures}, VCRL by group variance~\citep{jiang2025vcrl}, and FG-ExPO by a Gaussian at $p{=}0.5$~\citep{lin2026fgexpo}. Other work revives zero-variance groups~\citep{le2026noprompt,liu2025erpo,mao2026popo,baroian2026replay}, allocates rollout budgets~\citep{zou2026trace}, builds cold-start priors~\citep{sha2026thinkprior}, or balances explore--exploit portfolios~\citep{liang2026leeps}. \method{} keeps the rollout-free setting but derives coordinate, width, and target from GRPO's geometry.

\paragraph{Curricula and pacing.}
Automatic curricula favor intermediate difficulty or learning progress~\citep{bengio2009curriculum,graves2017automated,florensa2018goal,jiang2021plr,portelas2020survey}. For LLM reasoning, AdaRFT moves a target difficulty by reward feedback~\citep{shi2026adarft}, SEC runs a bandit over difficulty levels~\citep{chen2025sec}, learnability sampling targets $p(1{-}p)$~\citep{foster2025learnability}, and PCL finds prompts near $p{=}0.5$ with a value model~\citep{gao2025pcl}. \method{} instead paces an \emph{objective} along the pass$^k$--pass@$k$ family, limited by a closed-form prediction of waste.

\paragraph{What RLVR optimizes.}
Common RLVR algorithms ascend monotone transforms of the pass rate, $\arcsin\sqrt{p}$ for GRPO~\citep{davis2025objective}, and GRPO's normalization amplifies rare successes~\citep{mroueh2025grpo}; advantage reshaping can target pass@$k$~\citep{walder2025pkpo,chen2025passk,tang2025inference}. We show that the \emph{sampler} reshapes the objective too, without touching advantages and while saving the rollouts of zero-variance groups. Extended related work is in Appendix~\ref{app:related}.

\section{The Geometry of Pass Rates under GRPO}
\label{sec:geometry}

\paragraph{Setup.}
A prompt $q$ has a binary verifier reward and pass rate $p_q(\theta)=\Pr_{o\sim\pi_\theta(\cdot\mid q)}[r(q,o){=}1]$. GRPO draws $o_1,\dots,o_G\sim\pi_\theta(\cdot\mid q)$, forms $\hat A_{q,i}=(r_{q,i}-\bar r_q)/(s_q+\epsilon)$ from the group mean $\bar r_q$ and standard deviation $s_q$, and uses
\begin{equation}
\hat g_q=\frac{1}{G}\sum_{i=1}^{G}\hat A_{q,i}\,\nabla_\theta\log\pi_\theta(o_i\mid q)
\label{eq:grpo}
\end{equation}
inside a clipped surrogate~\citep{shao2024deepseekmath,schulman2017ppo}. A group with $m$ successes is \emph{informative} iff $0<m<G$; otherwise all advantages vanish. This happens with probability $1-z_G$, where $z_G(p)=p^G+(1-p)^G$ and $1-z_G=\mathrm{pass@}G-\mathrm{pass}^G$: a group teaches only if one of its $G$ attempts succeeds and one fails. A \emph{sampler} includes prompt $q$ in an update with probability $w_q$.

\paragraph{Arc length.}
Let $\psi=\arcsin\sqrt{p}\in[0,\hpi]$. The Fisher information of $\mathrm{Bernoulli}(\sin^2\psi)$ about $\psi$ is $4$ for every $\psi$, so the Fisher--Rao distance between pass rates is $2|\psi-\psi'|$ and the Jeffreys prior is uniform in $\psi$~\citep{rao1945information,jeffreys1946invariant,amari2016information}: $(\sqrt{p},\sqrt{1-p})$ traces a quarter circle and $\psi$ is its arc length. Proofs are in Appendix~\ref{app:proofs}.

\begin{proposition}[GRPO is uniform in arc length]
\label{prop:uniform}
For binary rewards and population-statistics normalization ($\epsilon\to0$),
$\E[\hat g_q]=2\,\omega_G(p_q)\,\nabla_\theta\psi_q$ with
$\omega_G(p)=\E_{m\sim\Bin(G,p)}\big[\sqrt{m(G-m)}\big]\big/\big(G\sqrt{p(1-p)}\big)$.
Moreover (i) $0\le\omega_G\le\sqrt{1-1/G}$; (ii) $\omega_G\to1$ uniformly on every compact subset of $(0,1)$ as $G\to\infty$; and (iii) $\omega_G(\psi)\le\sqrt{G-1}\,\min(\tan\psi,\cot\psi)$, with equality in the limits $\psi\to0$ and $\psi\to\hpi$.
\end{proposition}

\begin{corollary}[Every sampler is an objective]
\label{cor:objective}
If $w_q=w(\psi_q)$ is held fixed within an update, then $\sum_q w_q\E[\hat g_q]=\nabla_\theta\sum_q F_w(\psi_q)$ with $F_w(\psi)=2\int_0^{\psi}w(u)\,\omega_G(u)\,du$. Uniform sampling recovers the arcsine objective of~\citet{davis2025objective}; post-hoc filtering of zero-variance groups (DS) leaves $F_w$ unchanged and only rescales it; and a Gaussian $w=\N(\pst,\sigma^2)$ in arc length yields $F_w\approx\Phi\big((\psi-\pst)/\sigma\big)$, a smoothed count of prompts whose pass rate exceeds $\sin^2\pst$.
\end{corollary}

\begin{proposition}[Zero-variance groups are Gaussian boundary layers]
\label{prop:deadzone}
For all $\psi\in[0,\hpi]$, $z_G(\psi)=\cos^{2G}\psi+\sin^{2G}\psi\le e^{-G\psi^2}+e^{-G(\pi/2-\psi)^2}$. If $\psi\sim\N(\mu,P)$, then
$\E\big[e^{-G\psi^2}\big]=(1+2GP)^{-1/2}\exp\!\big(-G\mu^2/(1+2GP)\big)$, and symmetrically at $\hpi$.
\end{proposition}

\begin{proposition}[Evidence is homoscedastic in arc length]
\label{prop:noise}
Let $m\sim\Bin(n,\sin^2\psi)$. (i) For $0<m<n$, the Laplace approximation of the likelihood has mean $\arcsin\sqrt{m/n}$ and variance $1/(4n)$, independent of $m$; the Anscombe estimate $\hat\psi=\arcsin\sqrt{(m+3/8)/(n+3/4)}$ has variance $1/(4n)+O(n^{-2})$ uniformly on compact subsets~\citep{anscombe1948transformation}. (ii) For $m=0$ ($m=n$) the likelihood $\cos^{2n}\psi$ ($\sin^{2n}\psi$) is bounded by a Gaussian of mean $0$ ($\hpi$) and variance $1/(2n)$. (iii) By contrast, the delta-method variance of $\mathrm{logit}(\hat p)$, $1/(np(1-p))$, is unbounded as $p\to\{0,1\}$.
\end{proposition}

\begin{proposition}[Objectives are Gaussians in arc length]
\label{prop:kernel}
For real $k\ge1$, $H_k(\psi)=\tfrac{d}{d\psi}\big[1-(1-p)^k\big]=2k\sin\psi\cos^{2k-1}\psi$ is a log-concave probability density on $[0,\hpi]$ with mode $\pst_k=\arcsin(1/\sqrt{2k})$, i.e.\ $p^\star=1/(2k)$, and $-(\log H_k)''(\pst_k)=4k$. Its Laplace approximation is $\N(\pst_k,1/(4k))$; pass$^k$ is the mirror image with mode $\arccos(1/\sqrt{2k})$. Indexing the family by its mode gives the matched width $\sigma(\pst)=\min(\sin\pst,\cos\pst)/\sqrt2$.
\end{proposition}

Away from two ramps of width about $1/\sqrt{G}$, GRPO extracts the same arc-length signal from every prompt (Proposition~\ref{prop:uniform}), so \emph{where} to sample is a question about the objective and about waste. Proposition~\ref{prop:deadzone} turns waste into a Gaussian that integrates against any Gaussian belief, and Proposition~\ref{prop:noise} makes a linear--Gaussian filter the right belief model. Proposition~\ref{prop:kernel} fixes the kernel: its curvature $4k$ equals the Fisher information of $k$ Bernoulli draws, so an objective that counts $k$ attempts resolves pass rates exactly as $k$ samples do. The Laplace kernels match the exact ones within $0.05$--$0.07$ total variation for $k\in[1,16]$ (Appendix~\ref{app:numerics}).

\section{\method{}}
\label{sec:method}

\begin{figure*}[t]
\centering
\includegraphics[width=\textwidth]{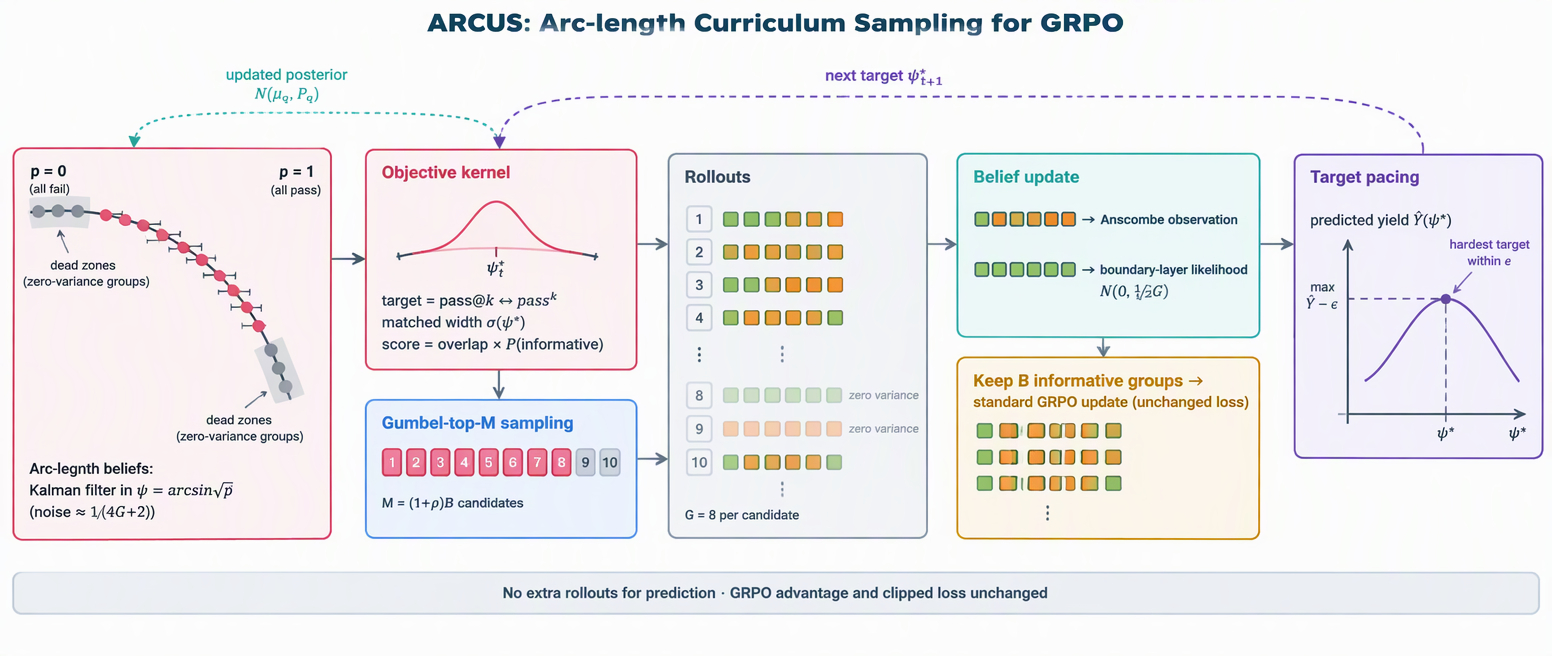}
\caption{\textbf{\method{} within one GRPO step.} Beliefs live on the arc $\psi=\arcsin\sqrt{p}$. An objective-matched Gaussian kernel at the paced target $\pst_t$ scores prompts; Gumbel-top-$M$ candidates are rolled out and update the beliefs; only informative groups enter the unchanged GRPO update; pacing moves $\pst_t$ to the hardest objective whose predicted yield stays within $\varepsilon$ of the best.}
\label{fig:overview}
\end{figure*}

\method{} replaces the prompt-sampling stage of GRPO (Figure~\ref{fig:overview}, Algorithm~\ref{alg:arcus}); the advantage in Eq.~\eqref{eq:grpo}, the clipped surrogate, and the optimizer are unchanged.

\paragraph{Arc-length beliefs.}
Each prompt carries a belief $\psi_q\sim\N(\mu_q,P_q)$, initialized at the Jeffreys moments $\N(\pi/4,\pi^2/48)$ and, after the first hundred observations, at the empirical moments of observed prompts. Before each update, every belief is propagated as
\begin{equation}
\mu_q\leftarrow\Pi\big[\mu_q+v_t\sin2\mu_q\big],\qquad P_q\leftarrow P_q+Q_t,
\label{eq:predict}
\end{equation}
where $\Pi$ projects onto $[0,\hpi]$. The drift $v_t$ models the improving policy: a shared competence gain $\delta x$ in logit space moves $\psi$ by $\tfrac{\delta x}{4}\sin2\psi$, so mobility vanishes at both ends (Appendix~\ref{app:drift}); the diffusion $Q_t$ absorbs prompt-specific change. Both are fitted online by moment matching on the innovations of revisited prompts~\citep{mehra1970identification}. An observed group with $m$ successes triggers a scalar Kalman update~\citep{kalman1960new}, $K=P_q/(P_q{+}R)$, $\mu_q\leftarrow\Pi[\mu_q+K(z-\mu_q)]$, $P_q\leftarrow(1-K)P_q$, whose observation follows Proposition~\ref{prop:noise}:
\begin{equation}
(z,R)=\begin{cases}
\big(\hat\psi_{\mathrm{A}}(m),\ \tfrac{1}{4G+2}\big) & 0<m<G,\\
\big(0,\ \tfrac{1}{2G}\big) & m=0,\\
\big(\hpi,\ \tfrac{1}{2G}\big) & m=G,
\end{cases}
\label{eq:obs}
\end{equation}
with $\hat\psi_{\mathrm{A}}$ the Anscombe estimate. A zero-variance group is thus not wasted evidence: its boundary-layer likelihood pulls the belief into the corresponding dead zone.

\paragraph{Objective-matched scoring.}
For a target $\pst$, the kernel is $\N(\pst,\sigma^2)$ with $\sigma=\sigma(\pst)$ from Proposition~\ref{prop:kernel}. Its overlap with the belief and the probability of an informative group (Proposition~\ref{prop:deadzone}) are both closed form; with $a_q=1+2GP_q$,
\begin{align}
\kappa_q(\pst)&=\sqrt{\tfrac{\sigma^2}{\sigma^2+P_q}}\;e^{-\frac{(\mu_q-\pst)^2}{2(\sigma^2+P_q)}},\label{eq:kappa}\\
\yhat_q&=1-\frac{e^{-G\mu_q^2/a_q}+e^{-G(\hpi-\mu_q)^2/a_q}}{\sqrt{a_q}}.\nonumber
\end{align}
The score $s_q(\pst)=\kappa_q(\pst)\,\yhat_q$ weighs the objective relevance of $q$ by the chance that its rollouts produce any gradient (Appendix~\ref{app:score}). Uncertainty adds $P_q$ to the kernel variance, which widens the window for rarely seen prompts and revisits them. \method{} draws $M=\lceil(1+\rho)B\rceil$ candidates by Gumbel-top-$M$ on $\log s_q/T$~\citep{kool2019gumbel}, i.e.\ without replacement with probability $\propto s_q^{1/T}$.

\paragraph{Frontier pacing.}
Harder targets push coverage but approach the lower dead zone, and Proposition~\ref{prop:deadzone} predicts this cost before it is paid. For each target $\psi$ on a grid $\Psi$ from pass@$G$ to pass$^G$ (the objectives a group of $G$ can express), we predict the candidate yield $\Yhat_t(\psi)=\tfrac1M\sum_q\pi_q(\psi)\yhat_q$, with capped-proportional inclusion probabilities $\pi_q=\min(1,c\,s_q(\psi)^{1/T})$, $\sum_q\pi_q=M$. The target is the hardest admissible one,
\begin{equation}
\tilde\psi_t=\min\big\{\psi\in\Psi:\Yhat_t(\psi)\ge\max\nolimits_{\psi'}\Yhat_t(\psi')-\varepsilon\big\},
\label{eq:pace}
\end{equation}
rate-limited by $|\pst_t-\pst_{t-1}|\le\Delta$ and held at pass@1 ($\pst{=}\pi/4$) during a warm-up pass over the pool. With sharp beliefs, Eq.~\eqref{eq:pace} settles where $z_G(\psi)=2^{1-G}+\varepsilon$, i.e.\ $\psi\approx\sqrt{\ln(1/(\varepsilon+2^{1-G}))/G}$, so the reachable objective hardens as $G$ grows, a testable prediction (\S\ref{sec:mechanism}). By Corollary~\ref{cor:objective}, pacing is a homotopy through objectives: GRPO ascends a smoothed count of prompts with $p_q>\sin^2\pst_t$, whose threshold hardens while waste stays bounded.

\paragraph{Post-rollout selection.}
All $M$ candidates update their beliefs; the informative groups, at most $B$ of them ranked by updated score, enter the GRPO update. Dropping zero-variance groups leaves the expected update unchanged (Corollary~\ref{cor:objective}) and removes their backward cost.

\begin{algorithm}[t]
\footnotesize
\DontPrintSemicolon
\SetAlgoLined
\caption{\method{}: one GRPO step}
\label{alg:arcus}
\KwInput{beliefs $\{(\mu_q,P_q)\}$; target $\pst_{t-1}$; batch $B$, group $G$; $\rho,\varepsilon,T,\Delta$}
Predict all beliefs with Eq.~\eqref{eq:predict}\;
\eIf{warm-up}{$\pst_t\leftarrow\pi/4$}{$\pst_t\leftarrow$ Eq.~\eqref{eq:pace}, clipped to $\pst_{t-1}\pm\Delta$}
$s_q\leftarrow\kappa_q(\pst_t)\,\yhat_q$ \tcp*{Eq.~\eqref{eq:kappa}}
$\mathcal{C}\leftarrow$ Gumbel-top-$M$ of $\log s_q/T$, $M=\lceil(1+\rho)B\rceil$\;
Generate $G$ rollouts for each $q\in\mathcal{C}$; record $m_q$\;
Kalman-update $\{(\mu_q,P_q)\}_{q\in\mathcal{C}}$ with Eq.~\eqref{eq:obs}; update $v_t,Q_t$\;
$\mathcal{B}\leftarrow$ informative groups in $\mathcal{C}$, top-$B$ by $s_q(\pst_t)$\;
GRPO update on $\mathcal{B}$ with Eq.~\eqref{eq:grpo} (unchanged)\;
\end{algorithm}

\paragraph{Cost.}
All quantities are closed-form scalars per prompt; pacing over $|\Psi|{=}41$ targets costs $O(N|\Psi|)$ vector operations, under 0.1\,s per step for $N{\approx}17$k. \method{} adds no forward pass, auxiliary model, or prediction rollout; its only extra rollouts are the $\rho BG$ candidate margin.

\paragraph{Existing selectors as special cases.}
The geometry also locates prior heuristics. DS is objective-neutral (Corollary~\ref{cor:objective}): it buys a cleaner batch in the arcsine direction with extra rollouts. Learnability sampling by $p(1-p)$~\citep{foster2025learnability} equals $\tfrac14H_1(\psi)^2$, a pass@1 kernel narrowed from $\sigma{=}1/2$ to $1/(2\sqrt2)$, and the $\sqrt{p(1-p)}$ score inside CurES's softmax~\citep{zeng2025cures} is $\tfrac12H_1$. A Gaussian over raw $p$~\citep{lin2026fgexpo} matches no member of the objective family: in arc length its width varies as $\sigma_p/\sin2\psi$, so its shape is set by the parameterization rather than by an objective. Logit filters~\citep{zhu2026kgps} are linear--Gaussian in the one coordinate where the noise of zero-variance evidence is unbounded (Proposition~\ref{prop:noise}). Each is \method{} with one of its three choices made differently: a fixed target, a mismatched width, or a mismatched coordinate.

\begin{table*}[t]
\centering
\footnotesize
\setlength{\tabcolsep}{5.2pt}
\renewcommand{\arraystretch}{0.96}
\input{tables/main_results.tex}
\caption{\textbf{Main results} (accuracy, \%, mean of three seeds; AIME mean@32, AMC23 mean@16, others mean@4). Groups: uniform, predict-then-select, evaluate-then-filter, ours. \textbf{Roll.}: generated rollouts relative to GRPO at equal policy updates; $\Delta$: average gain over GRPO. \textbf{Bold}/\underline{underline}: best/second best among trained models.}
\label{tab:main}
\end{table*}

\section{Experiments}
\label{sec:experiments}

\subsection{Setup}
\label{sec:setup}

\paragraph{Models, data, and training.}
We train Qwen2.5-Math-1.5B/7B~\citep{yang2024qwen25math} and Qwen3-4B-Base~\citep{yang2025qwen3} on DAPO-Math-17k~\citep{yu2025dapo} with binary rule-based rewards; since Qwen2.5-Math can respond to spurious signals~\citep{shao2025spurious}, Appendix~\ref{app:extra} adds a non-Qwen backbone, a long-CoT model, and a second pool. All methods share the GRPO loss~\citep{shao2024deepseekmath} in verl~\citep{sheng2025hybridflow} with vLLM~\citep{kwon2023vllm}: $B{=}256$, $G{=}8$, 300 policy updates, learning rate $10^{-6}$, clip $0.2$, no KL. They differ only in which prompts are rolled out and which groups enter the update, and are compared at matched updates. \method{} uses $\rho{=}0.25$, $\varepsilon{=}0.03$, $T{=}0.3$, $\Delta{=}0.005$\,rad, $|\Psi|{=}41$ throughout; \method{}$_{\rho=0}$ matches GRPO's rollout budget. We evaluate on AIME24/25, AMC23, MATH500~\citep{hendrycks2021math,lightman2024verify}, Minerva~\citep{lewkowycz2022minerva}, and OlympiadBench~\citep{he2024olympiadbench} (temperature $0.6$, top-$p$ $0.95$; details in Appendix~\ref{app:setup}).

\paragraph{Baselines.}
Uniform GRPO; evaluate-then-filter DS~\citep{yu2025dapo} and GRESO~\citep{zheng2025greso}, which fill the batch with informative groups at extra rollout cost; and the rollout-free GCS (Gaussian over EMA pass rates, $\mu{=}0.5$, $\sigma{=}0.35$)~\citep{lin2026fgexpo}, AdaRFT~\citep{shi2026adarft}, CurES~\citep{zeng2025cures}, MoPPS~\citep{qu2026mopps}, DPS~\citep{mao2026dps}, and KGPS~\citep{zhu2026kgps}, with recommended settings.

\subsection{Main Results}
\label{sec:main}

\method{} attains the best average on all three backbones (Table~\ref{tab:main}): $37.6$, $46.9$, and $49.2$, i.e.\ $+2.8$ to $+2.9$ over GRPO and $+1.1$ to $+1.2$ over DS, the strongest baseline, at $\times1.25$ GRPO's rollouts against $\times2.4$--$2.9$ for DS. First, the rollout-matched \method{}$_{\rho=0}$ already exceeds every predict-then-select method by \gainZeroMin--\gainZeroMax{} points and DS on every backbone, so the gain does not come from the candidate margin. Second, selectors that model uncertainty and change (DPS, KGPS) beat those that rank point estimates (GCS, AdaRFT); GCS, a Gaussian over raw pass rates, barely helps: a Gaussian works only in the right coordinate and width. Third, gains concentrate on hard benchmarks: on Qwen2.5-Math-7B, $+\aimeGainSeven$ on AIME24 and $+3.9$ on AMC23 but $+\mathGainSeven$ on MATH500, where DS stays best on every backbone, as expected from a target that spends the budget on prompts the policy cannot yet solve reliably.

\begin{table}[t]
\centering
\footnotesize
\setlength{\tabcolsep}{3.2pt}
\input{tables/cost.tex}
\caption{\textbf{Cost} on Qwen2.5-Math-7B (300 updates, 8$\times$H100): rollouts (M), tokens (B), yield (informative share of generated groups), Inf.\ (informative share of the update batch), wall-clock hours, and six-benchmark average.}
\label{tab:cost}
\end{table}

\paragraph{Efficiency.}
\method{} raises the informative share of generated groups from \yieldGRPO\% (GRPO) to \yieldArcus\% (\yieldArcusZero\% for \method{}$_{\rho=0}$), above the best predictive selector (KGPS, \yieldKGPS\%; Table~\ref{tab:cost}); DS gets a clean batch by discarding most of what it generates (\yieldDS\% yield). With a 25\% margin, \method{}'s update batch is \infArcus\% informative while it generates 48\% fewer rollouts and 50\% fewer tokens than DS and finishes in \hoursArcus\,h instead of \hoursDS\,h. Per unit of compute (Figure~\ref{fig:results}b), it reaches DS's final accuracy after 0.45M rollouts, \rollToDS\% fewer than DS, and every $\rho\in\{0,0.1,0.25,0.5\}$ lies above the baselines' cost--accuracy frontier (Figure~\ref{fig:results}d).

\begin{figure*}[t]
\centering
\includegraphics[width=\textwidth]{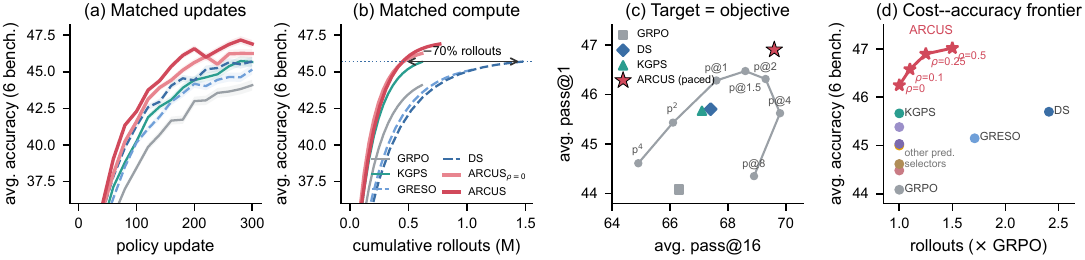}
\caption{\textbf{Accuracy, efficiency, and objective control} (Qwen2.5-Math-7B). (a) Six-benchmark average vs.\ policy updates ($\pm1$ s.d.). (b) The same runs vs.\ cumulative rollouts. (c) Fixed targets along the objective family (gray; p$^k$ = pass$^k$) trade pass@1 against pass@16; the paced target dominates. (d) Cost--accuracy plane; \method{}'s curve varies $\rho$.}
\label{fig:results}
\end{figure*}

\subsection{Ablations}
\label{sec:ablation}

Table~\ref{tab:ablation} removes one ingredient at a time (Qwen2.5-Math-7B). \emph{Geometry matters most}: the same pipeline with a Gaussian over raw pass rates loses $0.69$, and a logit-space Kalman belief loses $0.96$; the latter is over-confident on near-zero-variance prompts (Figure~\ref{fig:geometry}c) and its update batch is 7 points less informative. \emph{The dead-zone factor} $\yhat$ is the key kernel component ($-1.03$); without it hard targets pull candidates into the all-fail layer. A fixed width ($-0.42$) and ignoring belief variance ($-0.35$) matter less. \emph{Pacing beats fixed objectives}: holding pass@1 loses $0.62$, jumping straight to the final target loses $0.76$ because early beliefs cannot support it, and AdaRFT-style reward feedback loses $0.85$. Greedy selection and dropping post-rollout ranking lose $0.53$ and $0.29$. Appendix~\ref{app:extra} adds drift, prior, and observation ablations, the exact kernel ($46.87$ vs.\ $46.90$), and sensitivity to $\varepsilon$, $T$, $\rho$, and $G$.

\begin{table}[t]
\centering
\footnotesize
\setlength{\tabcolsep}{2.8pt}
\input{tables/ablation.tex}
\caption{\textbf{Ablations} on Qwen2.5-Math-7B (six-benchmark average; Inf.: informative share of the update batch).}
\label{tab:ablation}
\end{table}

\subsection{Mechanism Analysis}
\label{sec:mechanism}

\begin{figure*}[t]
\centering
\includegraphics[width=\textwidth]{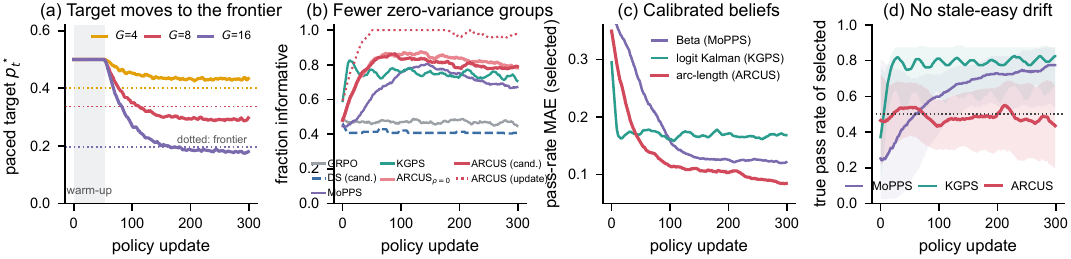}
\caption{\textbf{How \method{} works} (Qwen2.5-Math-7B). (a) Paced target $p^\star_t=\sin^2\pst_t$ for three group sizes (dotted: sharp-belief frontier). (b) Informative share of generated groups (dotted red: \method{}'s update batch). (c) Pass-rate error of the selected prompts. (d) Median and inter-quartile band of the \emph{true} pass rate of selected prompts.}
\label{fig:mechanism}
\end{figure*}

\paragraph{The target finds a $G$-dependent frontier.}
After warm-up the paced target leaves pass@1 and settles at $p^\star\approx0.29$ for $G{=}8$, $0.43$ for $G{=}4$, and $0.18$ for $G{=}16$ (Figure~\ref{fig:mechanism}a), within $0.05$ of the closed-form frontier. Larger groups thin the lower dead zone, so harder objectives become affordable: with $G{=}16$ the run reaches pass@$k$ with $k\approx2.8$ at no extra waste (accuracies in Appendix~\ref{app:extra}). The resulting easy-to-hard curriculum is emergent; nothing is scheduled by hand.

\paragraph{Beliefs stay calibrated and fresh.}
\method{} removes most zero-variance waste after the first pass over the pool (Figure~\ref{fig:mechanism}b). Its pass-rate error on selected prompts falls to $0.09$, versus $0.12$ for discounted Beta posteriors and $0.17$ for a logit Kalman filter (Figure~\ref{fig:mechanism}c), and its normalized innovations stay near one (Figure~\ref{fig:dynamics}c). Figure~\ref{fig:mechanism}d exposes a shared failure of history-based selectors: beliefs go stale in a predictable direction, because a prompt that looked intermediate when last seen is easier now, so MoPPS and KGPS drift toward prompts with true pass rates of $0.75$--$0.8$. \method{}'s drift anticipates this, keeping the selected median near $0.5$ while its lower quartile extends toward harder prompts.

\paragraph{The target is the objective.}
Freezing the target along the objective family trades pass@1 against pass@16 as Corollary~\ref{cor:objective} predicts (Figure~\ref{fig:results}c): pass$^k$ targets sharpen easy prompts and lose coverage, and pass@$k$ targets gain coverage until the dead zone eats the yield (pass@8: $47\%$ yield). The paced target beats every fixed objective on pass@1 and is within $0.2$ of the best on pass@16 ($\passSixteenArcus$ vs.\ $\passSixteenGRPO$ for GRPO), matching its higher entropy (Figure~\ref{fig:dynamics}).

\subsection{Generality}
\label{sec:generality}

The gains are not tied to one backbone, loss, or group size (Appendix~\ref{app:extra}). On the long-CoT DeepSeek-R1-Distill-Qwen-1.5B with an 8k limit, \method{} improves GRPO by $+2.4$ and DS by $+1.1$. On Llama-3.2-3B-Instruct, whose pool is dominated by all-fail groups and where DS needs $\times3.5$ rollouts, the gains are $+1.9$ and $+0.8$. With MATH-train instead of DAPO-Math-17k as the pool, they are $+2.2$ and $+0.9$. Replacing the GRPO loss by the DAPO token-level loss, Dr.~GRPO, or RLOO keeps a margin of $+1.1$ to $+1.2$ over DS, although for the last two the native coordinate is $p$ and the kernel must be reweighted (Appendix~\ref{app:drgrpo}). Across $G\in\{4,8,16\}$ the margin over DS is $+1.2$ to $+1.5$ while DS's overhead falls from $\times3.4$ to $\times2.0$, so the benefit of predicting waste instead of paying for it persists even when waste itself shrinks.

\section{Conclusion}
\label{sec:conclusion}

GRPO has a native geometry for pass rates, the Fisher--Rao arc $\psi=\arcsin\sqrt{p}$, on which its update is uniform, zero-variance groups are Gaussian boundary layers, evidence has constant noise, and pass@$k$ objectives are Gaussians. A prompt curriculum should therefore be a Gaussian in arc length whose center is an objective. \method{} realizes this as a rollout-free sampler with closed-form scoring and objective pacing that leaves GRPO's loss untouched and improves accuracy while generating half the rollouts of dynamic sampling. The geometry applies wherever group-normalized binary rewards appear, including code, agents, and test-time allocation.

\section*{Limitations}

\method{} is derived for binary rewards and the standard-deviation normalization of GRPO. Variants without that normalization, such as Dr.~GRPO, have a different native coordinate (Appendix~\ref{app:drgrpo}), and continuous rewards require binarization or a new variance-stabilizing map. The sampler--objective correspondence treats prompts as having separate gradients and ignores transfer between prompts; transfer is what makes the drift term necessary, and our global drift is a first-order model of it. The pacing rule encodes one preference: the hardest objective within a yield slack. Practitioners who value reliability over coverage may prefer another point on the objective family. Our experiments cover mathematical reasoning with models up to 7B parameters and $G\le16$; code, agentic tasks, and larger scales remain to be tested.

\section*{Ethics Statement}

\method{} changes only which training prompts receive rollouts, so it inherits the risks of the models and data it is applied to. By reducing generated rollouts and tokens, it lowers the energy cost of RLVR post-training. All datasets and benchmarks are public and used under their licenses; no human subjects or personal data are involved. AI assistants were used for language editing and for checking LaTeX; the method, proofs, and experimental design are the authors' own, and every reference was checked against its primary source.

\bibliography{custom}

\appendix
\input{appendix.tex}

\end{document}

%% file: tables/stats.tex
\newcommand{\gainGRPOmin}{2.8}
\newcommand{\gainGRPOmax}{2.9}
\newcommand{\gainDSmin}{1.1}
\newcommand{\gainDSmax}{1.2}
\newcommand{\gainZeroMin}{0.6}
\newcommand{\gainZeroMax}{0.7}
\newcommand{\saveDSmin}{48}
\newcommand{\saveDSmax}{57}

\newcommand{\rollToDS}{70}
\newcommand{\yieldGRPO}{47}
\newcommand{\yieldArcusZero}{81}
\newcommand{\yieldKGPS}{75}
\newcommand{\yieldDS}{41}
\newcommand{\infArcus}{97}
\newcommand{\yieldArcus}{79}
\newcommand{\hoursArcus}{24.1}
\newcommand{\hoursDS}{41.6}

\newcommand{\aimeGainSeven}{4.2}
\newcommand{\mathGainSeven}{1.4}
\newcommand{\passSixteenArcus}{69.6}
\newcommand{\passSixteenGRPO}{66.3}

%% file: tables/main_results.tex
\begin{tabular}{l c cccccc c c}
\toprule
\textbf{Method} & \textbf{Roll.} & \textbf{AIME24} & \textbf{AIME25} & \textbf{AMC23} & \textbf{MATH500} & \textbf{Minerva} & \textbf{Olymp.} & \textbf{Avg.} & $\bm{\Delta}$ \\
\midrule
\multicolumn{10}{c}{\textit{\textbf{Qwen2.5-Math-1.5B}} (DAPO-Math-17k, 300 updates)}\\
\midrule
{\color{gray}Base model} & {\color{gray}--} & {\color{gray}7.4} & {\color{gray}3.9} & {\color{gray}28.6} & {\color{gray}42.3} & {\color{gray}12.1} & {\color{gray}20.8} & {\color{gray}19.2} & {\color{gray}--} \\
GRPO & $\times$1.00 & 15.1 & 9.4 & 47.9 & 73.2 & 28.7 & 34.6 & 34.8 & -- \\
\cmidrule(lr){1-10}
GCS & $\times$1.00 & 15.4 & 9.8 & 48.3 & 73.4 & 28.9 & 34.9 & 35.1 & +0.3 \\
AdaRFT & $\times$1.00 & 15.8 & 9.6 & 48.6 & 73.9 & 28.6 & 35.2 & 35.3 & +0.5 \\
CurES & $\times$1.00 & 16.1 & 10.2 & 49.0 & 73.6 & 29.4 & 35.3 & 35.6 & +0.8 \\
MoPPS & $\times$1.00 & 16.3 & 10.0 & 49.4 & 74.1 & 29.0 & 35.6 & 35.7 & +0.9 \\
DPS & $\times$1.00 & 16.8 & 10.5 & 49.9 & 74.3 & 29.6 & 36.0 & 36.2 & +1.4 \\
KGPS & $\times$1.00 & 16.6 & 10.9 & 50.1 & 74.5 & 29.9 & 35.9 & 36.3 & +1.5 \\
\cmidrule(lr){1-10}
GRESO & $\times$1.97 & 16.0 & 10.3 & 49.2 & 74.4 & 29.3 & 35.8 & 35.8 & +1.0 \\
DS & $\times$2.94 & 16.9 & 10.6 & 50.4 & \textbf{75.3} & 29.5 & 36.3 & 36.5 & +1.7 \\
\cmidrule(lr){1-10}
\rowcolor{oursbg}\method{}$_{\rho=0}$ & $\times$1.00 & \underline{17.6} & \underline{11.4} & \underline{51.2} & 74.7 & \underline{30.1} & \underline{36.8} & \underline{37.0} & +2.1 \\
\rowcolor{oursbg}\method{} & $\times$1.25 & \textbf{18.5} & \textbf{11.9} & \textbf{52.3} & \underline{75.2} & \textbf{30.6} & \textbf{37.4} & \textbf{37.6} & +2.8 \\
\midrule
\multicolumn{10}{c}{\textit{\textbf{Qwen2.5-Math-7B}} (DAPO-Math-17k, 300 updates)}\\
\midrule
{\color{gray}Base model} & {\color{gray}--} & {\color{gray}13.2} & {\color{gray}6.1} & {\color{gray}40.3} & {\color{gray}58.4} & {\color{gray}16.5} & {\color{gray}26.1} & {\color{gray}26.8} & {\color{gray}--} \\
GRPO & $\times$1.00 & 30.4 & 13.1 & 62.8 & 80.4 & 36.3 & 41.5 & 44.1 & -- \\
\cmidrule(lr){1-10}
GCS & $\times$1.00 & 30.9 & 13.8 & 63.2 & 80.7 & 36.5 & 41.8 & 44.5 & +0.4 \\
AdaRFT & $\times$1.00 & 31.2 & 13.6 & 63.5 & 80.8 & 36.6 & 42.0 & 44.6 & +0.5 \\
CurES & $\times$1.00 & 31.5 & 14.3 & 63.9 & 80.9 & 37.2 & 42.2 & 45.0 & +0.9 \\
MoPPS & $\times$1.00 & 31.8 & 14.1 & 64.0 & 81.0 & 36.9 & 42.4 & 45.0 & +0.9 \\
DPS & $\times$1.00 & 32.1 & 14.5 & 64.8 & 81.2 & 37.0 & 42.7 & 45.4 & +1.3 \\
KGPS & $\times$1.00 & 32.4 & 14.8 & 64.6 & 81.3 & \textbf{38.0} & 42.9 & 45.7 & +1.6 \\
\cmidrule(lr){1-10}
GRESO & $\times$1.71 & 31.9 & 14.0 & 64.3 & 81.4 & 36.7 & 42.6 & 45.2 & +1.1 \\
DS & $\times$2.41 & 32.6 & 14.6 & 64.9 & \textbf{81.9} & 37.1 & 43.1 & 45.7 & +1.6 \\
\cmidrule(lr){1-10}
\rowcolor{oursbg}\method{}$_{\rho=0}$ & $\times$1.00 & \underline{33.5} & \underline{15.4} & \underline{65.9} & 81.6 & 37.5 & \underline{43.6} & \underline{46.2} & +2.2 \\
\rowcolor{oursbg}\method{} & $\times$1.25 & \textbf{34.6} & \textbf{16.2} & \textbf{66.7} & \underline{81.8} & \underline{37.8} & \textbf{44.3} & \textbf{46.9} & +2.8 \\
\midrule
\multicolumn{10}{c}{\textit{\textbf{Qwen3-4B-Base}} (DAPO-Math-17k, 300 updates)}\\
\midrule
{\color{gray}Base model} & {\color{gray}--} & {\color{gray}9.8} & {\color{gray}7.5} & {\color{gray}37.6} & {\color{gray}64.2} & {\color{gray}25.8} & {\color{gray}33.4} & {\color{gray}29.7} & {\color{gray}--} \\
GRPO & $\times$1.00 & 24.6 & 20.3 & 61.5 & 83.6 & 39.4 & 48.2 & 46.3 & -- \\
\cmidrule(lr){1-10}
GCS & $\times$1.00 & 25.1 & 20.6 & 61.9 & 83.8 & 39.7 & 48.4 & 46.6 & +0.3 \\
AdaRFT & $\times$1.00 & 25.4 & 20.9 & 62.3 & 83.7 & 39.5 & 48.8 & 46.8 & +0.5 \\
CurES & $\times$1.00 & 25.9 & 21.4 & 62.6 & 84.1 & 40.1 & 49.0 & 47.2 & +0.9 \\
MoPPS & $\times$1.00 & 26.1 & 21.2 & 63.0 & 84.2 & 39.8 & 49.3 & 47.3 & +1.0 \\
DPS & $\times$1.00 & 26.5 & 21.9 & 63.4 & 84.3 & 40.3 & 49.6 & 47.7 & +1.4 \\
KGPS & $\times$1.00 & 26.8 & 21.7 & 63.7 & 84.6 & 40.6 & 49.5 & 47.8 & +1.5 \\
\cmidrule(lr){1-10}
GRESO & $\times$1.83 & 25.8 & 21.6 & 62.8 & 84.5 & 40.0 & 49.4 & 47.3 & +1.1 \\
DS & $\times$2.63 & 26.9 & 22.1 & 63.6 & \textbf{85.3} & 40.2 & 49.9 & 48.0 & +1.7 \\
\cmidrule(lr){1-10}
\rowcolor{oursbg}\method{}$_{\rho=0}$ & $\times$1.00 & \underline{27.7} & \underline{22.8} & \underline{64.5} & 84.8 & \underline{40.9} & \underline{50.3} & \underline{48.5} & +2.2 \\
\rowcolor{oursbg}\method{} & $\times$1.25 & \textbf{28.9} & \textbf{23.6} & \textbf{65.4} & \underline{85.1} & \textbf{41.3} & \textbf{50.9} & \textbf{49.2} & +2.9 \\
\bottomrule
\end{tabular}

%% file: tables/cost.tex
\begin{tabular}{l cc cc c c}
\toprule
\textbf{Method} & \textbf{Roll.} & \textbf{Tok.} & \textbf{Yield} & \textbf{Inf.} & \textbf{Hours} & \textbf{Avg.} \\
 & (M) & (B) & (\%) & (\%) & & \\
\midrule
GRPO & 0.61 & 0.50 & 47 & 47 & 19.8 & 44.1 \\
MoPPS & 0.61 & 0.52 & 69 & 69 & 19.9 & 45.0 \\
KGPS & 0.61 & 0.52 & 75 & 75 & 20.0 & 45.7 \\
GRESO & 1.05 & 0.93 & 58 & 100 & 30.7 & 45.2 \\
DS & 1.48 & 1.33 & 41 & 100 & 41.6 & 45.7 \\
\rowcolor{oursbg}\method{}$_{\rho=0}$ & 0.61 & 0.53 & 81 & 81 & 20.0 & 46.2 \\
\rowcolor{oursbg}\method{} & 0.77 & 0.66 & 79 & 97 & 24.1 & 46.9 \\
\bottomrule
\end{tabular}

%% file: tables/ablation.tex
\begin{tabular}{l c c c}
\toprule
\textbf{Configuration} & \textbf{Avg.} & $\bm{\Delta}$ & \textbf{Inf.\,(\%)} \\
\midrule
\rowcolor{oursbg}\method{} (full) & \textbf{46.90} & -- & 97 \\
\midrule
\multicolumn{4}{l}{\textit{Geometry and belief}}\\
\quad Gaussian over raw $p$ & 46.21 & $-0.69$ & 93 \\
\quad Logit-space Kalman belief & 45.94 & $-0.96$ & 90 \\
\quad Discounted Beta belief & 46.02 & $-0.88$ & 91 \\
\midrule
\multicolumn{4}{l}{\textit{Objective kernel}}\\
\quad Fixed width $\sigma{=}0.20$ rad & 46.48 & $-0.42$ & 96 \\
\quad No belief variance in kernel & 46.55 & $-0.35$ & 95 \\
\quad No dead-zone factor $\hat y$ & 45.87 & $-1.03$ & 84 \\
\midrule
\multicolumn{4}{l}{\textit{Target pacing}}\\
\quad Fixed pass@1 target & 46.28 & $-0.62$ & 98 \\
\quad Fixed frontier target ($p^\star{=}0.29$) & 46.14 & $-0.76$ & 92 \\
\quad Reward-feedback pacing & 46.05 & $-0.85$ & 93 \\
\midrule
\multicolumn{4}{l}{\textit{Sampling and selection}}\\
\quad No post-rollout selection & 46.61 & $-0.29$ & 97 \\
\quad Greedy top-$M$ ($T{=}0$) & 46.37 & $-0.53$ & 95 \\
\bottomrule
\end{tabular}

%% file: appendix.tex

\section{Extended Related Work}
\label{app:related}

\paragraph{RLVR and group-relative estimators.}
GRPO~\citep{shao2024deepseekmath} made critic-free RL the default for reasoning models~\citep{deepseekai2025r1,kimi2025k15,he2025skywork}. Its variants change the normalization, the clipping, or the aggregation: DAPO decouples clipping and adds dynamic sampling~\citep{yu2025dapo}, Dr.~GRPO removes the standard-deviation and length normalizations~\citep{liu2025understanding}, REINFORCE++ normalizes globally~\citep{hu2025reinforcepp}, GSPO moves the importance ratio to the sequence level~\citep{zheng2025gspo}, and RLOO uses leave-one-out baselines~\citep{ahmadian2024back,williams1992simple}. Prolonged training~\citep{liu2025prorl} and entropy management~\citep{cui2025entropy} target the exploration collapse that RLVR can cause~\citep{yue2025does}, and negative samples help preserve coverage~\citep{zhu2025surprising}. \method{} changes none of these estimators; it controls which prompts they see. Appendix~\ref{app:drgrpo} shows how the geometry changes when the normalization is removed.

\paragraph{What the choice of estimator optimizes.}
\citet{davis2025objective} show that REINFORCE, rejection sampling, and GRPO with binary rewards perform stochastic gradient ascent on $h(p)$ for $h$ equal to the identity, approximately $\log$, and approximately $2\arcsin\sqrt{\cdot}$, respectively. They derive the exact finite-group transform and a Bernstein-polynomial recipe for choosing advantages to target any $h$. \citet{mroueh2025grpo} analyze GRPO's effective loss and show how normalization amplifies rare successes. BNPO normalizes by a Beta density~\citep{xiao2025bnpo}, and pass@$k$ objectives can be optimized by transforming rewards or advantages~\citep{walder2025pkpo,chen2025passk,tang2025inference}; the pass@$k$ metric follows \citet{chen2021codex} and its reliability counterpart pass$^k$ follows \citet{yao2024taubench}. The pairwise $U$-statistic view of GRPO~\citep{zhou2026ustat} is complementary. Our Corollary~\ref{cor:objective} is the sampler-side analogue of \citet{davis2025objective}: prompt sampling reshapes $h$ without touching the advantage and, unlike advantage reshaping, can avoid paying for zero-variance groups.

\paragraph{Prompt selection and zero-variance prompts.}
Beyond the methods discussed in \S\ref{sec:related}, offline curation selects small, informative prompt sets~\citep{li2025limr,wang2025oneshot}. Pairing hard-but-solvable prompts with easy-but-brittle ones amplifies rare events~\citep{pang2026beyond}, and industrial systems reuse zero-variance prompts at scale~\citep{zeng2025eachprompt}. Zero-variance queries can also be recycled in agentic search~\citep{coelho2026recycling}, and the same efficiency pressure motivates rollout-allocation methods that vary the group size per prompt or per prefix~\citep{nguyen2026vip,wang2026hora,jiang2026vigor,zou2026trace,hu2026duet} and budget-aware exploration~\citep{li2025knapsack}. These methods decide \emph{how many} rollouts a prompt receives; \method{} decides \emph{which} prompts receive them, and the two compose. Among predictive selectors, the closest to \method{} are KGPS~\citep{zhu2026kgps}, which also filters beliefs with a Kalman filter, and FG-ExPO~\citep{lin2026fgexpo}, which also weights prompts by a Gaussian. KGPS works in logit space with a zero-mean random walk and plug-in delta-method noise, and scores by $\E[p(1-p)]$ with quadrature. FG-ExPO uses a fixed Gaussian ($\mu{=}0.5,\sigma{=}0.35$) over raw EMA pass rates. \method{} differs in the coordinate (arc length, where noise is constant and dead zones are Gaussian), the observation model (boundary-layer likelihoods for zero-variance groups), the dynamics (mobility-weighted drift), the kernel (width derived from the objective, closed-form convolution with the belief), and the target (paced along the objective family instead of fixed).

\paragraph{Curricula.}
Classical curricula order examples from easy to hard~\citep{bengio2009curriculum}; automatic curricula choose tasks by learning progress~\citep{graves2017automated,portelas2020survey}, intermediate success probability~\citep{florensa2018goal}, or replay scores~\citep{jiang2021plr}. LLM-specific curricula include static easy-to-hard schedules~\citep{parashar2025curriculum}, bandits over difficulty levels~\citep{chen2025sec}, reward-feedback targets~\citep{shi2026adarft}, learnability sampling~\citep{foster2025learnability}, and value-model filtering~\citep{gao2025pcl}. In \method{} the easy-to-hard pattern emerges from a waste constraint rather than being scheduled, and the ``difficulty'' being paced is an objective in a principled family.

\paragraph{Statistics and filtering.}
Variance-stabilizing transforms for binomial data go back to \citet{anscombe1948transformation}. The Fisher--Rao metric and the Jeffreys prior are classical~\citep{rao1945information,jeffreys1946invariant,amari2016information}; so are Kalman filtering~\citep{kalman1960new}, innovation-based noise identification~\citep{mehra1970identification}, Laplace approximations~\citep{tierney1986accurate}, and Gumbel-top-$k$ sampling~\citep{kool2019gumbel}. Our contribution is to observe that GRPO's normalization selects exactly this geometry, and to derive a sampler in which these classical tools compose in closed form.

\section{Proofs}
\label{app:proofs}

Throughout, $p=\sin^2\psi$, $\psi\in[0,\hpi]$, $m\sim\Bin(G,p)$ is the number of successes in a group, and $C(q)$ is the set of correct responses to $q$, so $p=\sum_{o\in C(q)}\pi_\theta(o\mid q)$ and $\nabla p=\sum_{o\in C(q)}\nabla\pi_\theta(o\mid q)$. We use $\frac{dp}{d\psi}=\sin2\psi=2\sqrt{p(1-p)}$.

\subsection{Proof of Proposition~\ref{prop:uniform}}
\label{app:proof-uniform}

\emph{Conditional scores.} Given the reward labels, the $G$ responses are independent, and a response labelled correct is distributed as $\pi_\theta(\cdot\mid q)$ restricted to $C(q)$ and renormalized. Hence
\begin{align*}
\E[\nabla\log\pi_\theta(o)\mid r=1]&=\tfrac{1}{p}\textstyle\sum_{o\in C}\nabla\pi_\theta(o)=\tfrac{\nabla p}{p},\\
\E[\nabla\log\pi_\theta(o)\mid r=0]&=-\tfrac{\nabla p}{1-p},
\end{align*}
where the second identity uses $\sum_o\nabla\pi_\theta(o)=0$.

\emph{Conditional advantages.} With population statistics, $\bar r=m/G$ and $s=\sqrt{m(G-m)}/G$. If $m\in\{0,G\}$, every numerator in Eq.~\eqref{eq:grpo} is zero and $\hat g_q=0$. If $0<m<G$, correct responses receive $(1-m/G)/s$ and incorrect ones $-(m/G)/s$. Therefore
\begin{align*}
\E[\hat g_q\mid m]
&=\tfrac{1}{G}\Big[\tfrac{m(G-m)}{Gs}\tfrac{\nabla p}{p}+\tfrac{m(G-m)}{Gs}\tfrac{\nabla p}{1-p}\Big]\\
&=\tfrac{m(G-m)}{G^2 s}\cdot\tfrac{\nabla p}{p(1-p)}\\
&=\tfrac{\sqrt{m(G-m)}}{G}\cdot\tfrac{\nabla p}{p(1-p)}.
\end{align*}
Taking expectations over $m$ and substituting $\nabla p=2\sqrt{p(1-p)}\,\nabla\psi$ gives $\E[\hat g_q]=2\,\omega_G(p)\nabla\psi$ with $\omega_G$ as stated. The expression coincides with the finite-group weight of \citet{davis2025objective} after the change of variables; we verified the identity numerically to machine precision (Table~\ref{tab:numerics}).

\emph{(i)} By Jensen's inequality, $\E\sqrt{m(G-m)}\le\sqrt{\E[m(G-m)]}=\sqrt{G(G-1)p(1-p)}$, since $\E[m(G-m)]=G\E m-\E m^2=G(G-1)p(1-p)$. Dividing by $G\sqrt{p(1-p)}$ gives $\omega_G\le\sqrt{(G-1)/G}$.

\emph{(ii)} The map $x\mapsto\sqrt{x(1-x)}$ is bounded and uniformly continuous on $[0,1]$, and $m/G\to p$ in probability uniformly in $p$ (Chebyshev: $\Pr(|m/G-p|>\delta)\le1/(4G\delta^2)$). Hence $\E\sqrt{(m/G)(1-m/G)}\to\sqrt{p(1-p)}$ uniformly on $[0,1]$, and the ratio converges to one uniformly on any compact subset of $(0,1)$, where $\sqrt{p(1-p)}$ is bounded away from zero.

\emph{(iii)} For $1\le m\le G-1$ we have $\sqrt{m(G-m)}\le m\sqrt{G-1}$, because $G-m\le m(G-1)\iff G\le mG$. Thus $\E\sqrt{m(G-m)}\le\sqrt{G-1}\,Gp$ and $\omega_G\le\sqrt{G-1}\sqrt{p/(1-p)}=\sqrt{G-1}\tan\psi$. Exchanging the roles of successes and failures gives the $\cot\psi$ bound. As $p\to0$, $\E\sqrt{m(G-m)}=Gp(1-p)^{G-1}\sqrt{G-1}+O(p^2)$, so the ratio of $\omega_G$ to the bound tends to one; the case $p\to1$ is symmetric. \qed

\begin{remark}[Implementation details that do not change the geometry]
Bessel-corrected standard deviations multiply $\omega_G$ by the constant $\sqrt{(G-1)/G}$. A positive $\epsilon$ in the denominator shrinks each informative group by $s/(s+\epsilon)$, which interpolates between $\omega_G$ ($\epsilon\to0$) and the REINFORCE weight on $\nabla p$ ($\epsilon\to\infty$), as described by \citet{davis2025objective}; at the usual $\epsilon=10^{-6}$ the difference is invisible. Token-level aggregation multiplies each response's contribution by a length-dependent positive factor; it changes the per-prompt scale but not the fact that zero-variance groups contribute nothing, which is all that Propositions~\ref{prop:deadzone}--\ref{prop:kernel} use.
\end{remark}

\subsection{Proof of Corollary~\ref{cor:objective}}
\label{app:proof-objective}

If the inclusion probability of prompt $q$ is $w(\psi_q)$ and is not differentiated during the update, then by Proposition~\ref{prop:uniform} and the chain rule
\begin{align*}
\textstyle\sum_q w(\psi_q)\E[\hat g_q]
&=\textstyle\sum_q 2w(\psi_q)\omega_G(\psi_q)\nabla\psi_q\\
&=\textstyle\sum_q F_w'(\psi_q)\nabla\psi_q\\
&=\textstyle\nabla\sum_q F_w(\psi_q).
\end{align*} Uniform sampling ($w\equiv c$) gives $F_w=2c\int_0^\psi\omega_G\approx2c\psi$ away from the ramps: the arcsine objective.

\emph{Dynamic sampling.} DS draws prompts uniformly, generates their groups, and keeps informative ones until $B$ are collected. The kept contribution of a drawn prompt is $\hat g_q\mathbb{1}\{0<m<G\}=\hat g_q$, because $\hat g_q=0$ on zero-variance groups. The number of drawn prompts is a stopping time with respect to the i.i.d.\ sequence of draws, so by Wald's identity the expected sum of kept contributions equals $\E[\#\text{drawn}]\cdot\frac1N\sum_q\E[\hat g_q]$. This is the uniform direction multiplied by $\E[\#\text{drawn}]/B$. DS therefore buys a larger step in the arcsine direction, not a different objective.

\emph{Gaussian samplers.} If $w(\psi)\propto\exp(-(\psi-\pst)^2/(2\sigma^2))$ and $\omega_G\approx1$ on its support, then $F_w'\propto\N(\psi;\pst,\sigma^2)$ and $F_w\approx\Phi((\psi-\pst)/\sigma)+\text{const}$. Summed over prompts, this is a smoothed count of prompts with $\psi_q>\pst$, i.e.\ with $p_q>\sin^2\pst$.

\emph{Post-rollout selection.} If a group with $m$ successes is kept with probability $a(m)$, the same computation replaces $\sqrt{m(G-m)}$ by $a(m)\sqrt{m(G-m)}$ inside $\omega_G$. The result is again of the form $2\tilde w(\psi)\tilde\omega_G(\psi)\nabla\psi$, so post-selection modifies the implicit objective smoothly and never adds signal from zero-variance groups. \qed

\subsection{Proof of Proposition~\ref{prop:deadzone}}
\label{app:proof-deadzone}

Let $f(x)=\log\cos x+x^2/2$ on $[0,\hpi)$. Then $f(0)=0$ and $f'(x)=x-\tan x\le0$, so $\cos x\le e^{-x^2/2}$ and $\cos^{2G}\psi\le e^{-G\psi^2}$. Since $\sin\psi=\cos(\hpi-\psi)$, also $\sin^{2G}\psi\le e^{-G(\pi/2-\psi)^2}$. For $\psi\sim\N(\mu,P)$, completing the square gives
\begin{align*}
\int\tfrac{e^{-(\psi-\mu)^2/(2P)}}{\sqrt{2\pi P}}\,e^{-G\psi^2}d\psi
=\tfrac{1}{\sqrt{1+2GP}}\,e^{-\frac{G\mu^2}{1+2GP}},
\end{align*}
and symmetrically at $\hpi$. When the belief is supported on $[0,\hpi]$, the bound is pointwise, so $\yhat_q$ in Eq.~\eqref{eq:kappa} is a \emph{lower} bound on the belief-averaged probability that the group is informative; \method{} is conservative about waste. \qed

\subsection{Proof of Proposition~\ref{prop:noise}}
\label{app:proof-noise}

\emph{Fisher information.} For $\mathrm{Bernoulli}(p)$, $I(p)=1/(p(1-p))$. Reparameterizing by $\psi$ gives $I(\psi)=(dp/d\psi)^2I(p)=\sin^22\psi/(\sin^2\psi\cos^2\psi)=4$. The Jeffreys prior $\propto\sqrt{I(\psi)}$ is therefore uniform on $[0,\hpi]$, equivalently $\mathrm{Beta}(\tfrac12,\tfrac12)$ in $p$.

\emph{(i) Laplace likelihood.} The log-likelihood of $m$ successes in $n$ trials is $\ell(\psi)=2m\log\sin\psi+2(n-m)\log\cos\psi$. Its stationary point satisfies $m\cot\psi=(n-m)\tan\psi$, i.e.\ $\sin^2\psi=m/n$, and
$\ell''(\psi)=-2m\csc^2\psi-2(n-m)\sec^2\psi=-2n-2n=-4n$ there. The Laplace approximation~\citep{tierney1986accurate} is thus $\N(\arcsin\sqrt{m/n},1/(4n))$ for every $0<m<n$. For the Anscombe estimate $\hat\psi=h(\tilde p)$ with $h(x)=\arcsin\sqrt{x}$ and $\tilde p=(m+3/8)/(n+3/4)$, the delta method gives $\mathrm{Var}[\hat\psi]=h'(p)^2\,\mathrm{Var}[\tilde p]+O(n^{-2})=\frac{1}{4p(1-p)}\cdot\frac{np(1-p)}{(n+3/4)^2}+O(n^{-2})=\frac{1}{4n}+O(n^{-2})$. The remainder is uniform on compact subsets of $(0,1)$ because $h$ has bounded derivatives there. Anscombe's second-order analysis motivates the refinement $1/(4n+2)$, which Table~\ref{tab:numerics} confirms is accurate to $1\%$ for $p\in[0.2,0.8]$ at $n=8$.

\emph{(ii) Boundary layers.} The likelihood of $m=0$ is $(1-p)^n=\cos^{2n}\psi\le e^{-n\psi^2}$ by the proof of Proposition~\ref{prop:deadzone}, an unnormalized Gaussian with mean $0$ and variance $1/(2n)$. The case $m=n$ is symmetric.

\emph{(iii) Logit.} The delta method gives $\mathrm{Var}[\mathrm{logit}\,\hat p]\approx(p(1-p))^{-2}\cdot p(1-p)/n=1/(np(1-p))$, which diverges at both ends. In practice logit filters must clip $\hat p$ and substitute a plug-in $p$, which makes their observation model wrong exactly for near-zero-variance prompts (Figure~\ref{fig:geometry}c). \qed

\subsection{Proof of Proposition~\ref{prop:kernel}}
\label{app:proof-kernel}

$\frac{d}{d\psi}[1-\cos^{2k}\psi]=2k\sin\psi\cos^{2k-1}\psi=H_k(\psi)\ge0$, and $\int_0^{\pi/2}H_k=[1-\cos^{2k}\psi]_0^{\pi/2}=1$, so $H_k$ is a density. Equivalently, it is the density of $\arcsin\sqrt{X}$ for $X\sim\mathrm{Beta}(1,k)$, because $\Pr(X\le x)=1-(1-x)^k$. Differentiating $\log H_k=\log(2k)+\log\sin\psi+(2k-1)\log\cos\psi$ twice gives $(\log H_k)''=-\csc^2\psi-(2k-1)\sec^2\psi<0$, so $H_k$ is log-concave and unimodal. The mode solves $\cot\psi=(2k-1)\tan\psi$, i.e.\ $\tan^2\psi=1/(2k-1)$ and $\sin^2\psi=1/(2k)$. There $\csc^2\psi=2k$ and $\sec^2\psi=2k/(2k-1)$, so $(\log H_k)''=-2k-2k=-4k$. The pass$^k$ objective $\sin^{2k}\psi$ has derivative $2k\sin^{2k-1}\psi\cos\psi=H_k(\hpi-\psi)$. For a mode $\pst\le\pi/4$, $\sin^2\pst=1/(2k)$ gives $1/(2\sqrt k)=\sin\pst/\sqrt2$; for $\pst\ge\pi/4$ the mirror gives $\cos\pst/\sqrt2$. Hence $\sigma(\pst)=\min(\sin\pst,\cos\pst)/\sqrt2$. \qed

\begin{remark}[Large $k$]
Since $kX\to\mathrm{Exp}(1)$ for $X\sim\mathrm{Beta}(1,k)$, $\sqrt{k}\,\arcsin\sqrt X\to\sqrt{E}$ with $E\sim\mathrm{Exp}(1)$: the pass@$k$ kernel tends to a Rayleigh law with scale $1/\sqrt{2k}$. Its mode matches $\pst_k$, and its variance $(4-\pi)/(4k)$ is close to the Laplace value $1/(4k)$. The Gaussian is therefore accurate at both ends of the family, with a mild right skew that the dead-zone factor $\yhat$ further suppresses.
\end{remark}

\begin{remark}[Learnability sampling]
Sampling by $p(1-p)$~\citep{foster2025learnability} corresponds in arc length to $\tfrac14\sin^22\psi=\tfrac14H_1(\psi)^2$, the \emph{square} of the pass@1 kernel: a pass@1 target whose width is shrunk from $1/2$ to $1/(2\sqrt2)$. The score $\sqrt{p(1-p)}$ that CurES passes through a softmax~\citep{zeng2025cures} equals $\tfrac12H_1$. Both are kernels at a fixed pass@1 target with a width that is not matched to the objective.
\end{remark}

\subsection{The Drift Model}
\label{app:drift}

Suppose the policy's competence on a prompt is a logit $x$ with $p=\sigma(x)$. Then $\frac{d\psi}{dx}=\frac{d\psi}{dp}\frac{dp}{dx}=\frac{p(1-p)}{2\sqrt{p(1-p)}}=\frac{\sqrt{p(1-p)}}{2}=\frac{\sin2\psi}{4}$. A competence gain $\delta x$ shared across prompts, which is what transfer from training produces, therefore moves every belief by $\tfrac{\delta x}{4}\sin2\psi$. This gives the drift term $v_t\sin2\mu_q$ of Eq.~\eqref{eq:predict}, whose mobility vanishes in both dead zones. For an un-revisited prompt, the drift is what prevents its belief from going stale: without it, a prompt last seen at $p=0.5$ would be believed intermediate long after it has become easy (Figure~\ref{fig:mechanism}d).

\subsection{Why the Score Is Kernel Times Yield}
\label{app:score}

Consider the objective $J=\sum_qF(\psi_q)$ with $F'$ equal to the kernel $\N(\pst,\sigma^2)$. Under the separate-gradient approximation $\langle\nabla\psi_q,\nabla\psi_{q'}\rangle=c\,\mathbb{1}\{q=q'\}$, including prompt $q$ in an update of size $\eta$ increases $J$ to first order by $2\eta c\,F'(\psi_q)\,\omega_G(\psi_q)$. Averaging over the belief and treating the kernel and the ramp as approximately uncorrelated under the belief yields $\propto\kappa_q\cdot\E[\omega_G]$. \method{} uses the informative probability $\yhat_q$ in place of $\E[\omega_G]$. The two agree up to a bounded factor in the interior and differ inside the ramps, where $\omega_G/(1-z_G)\approx\sqrt{G-1}/(G\psi)$ grows because GRPO amplifies lone successes~\citep{mroueh2025grpo}. We prefer $\yhat$ for three reasons: the amplified signal of a lone success is also its noisiest component, $\yhat$ has the closed form needed for pacing, and it measures exactly the waste the method targets. The empirical difference is small ($46.71$ with $\omega_G$ vs.\ $46.90$; Table~\ref{tab:ablation-appx}).

\subsection{The Sharp-Belief Frontier}
\label{app:frontier}

Assume beliefs are exact and the candidate set concentrates at the target, the idealization in which the predicted yield of target $\psi$ is $1-z_G(\psi)$. The maximum over $\psi$ is attained at $\pi/4$ with value $1-2^{1-G}$. The hardest admissible target of Eq.~\eqref{eq:pace} therefore solves $z_G(\psi)=2^{1-G}+\varepsilon$ on $[0,\pi/4]$. There $z_G(\psi)\le e^{-G\psi^2}+e^{-G(\pi/2-\psi)^2}$ and the second term is at most $e^{-G\pi^2/16}$, so $\psi\approx\sqrt{\ln(1/(\varepsilon+2^{1-G}))/G}$ and the reachable objective is pass@$k$ with $k=1/(2\sin^2\psi)$. For $\varepsilon=0.03$, the exact roots are $p^\star=0.40$, $0.34$, and $0.20$ for $G=4$, $8$, and $16$ (Table~\ref{tab:group}). Finite belief precision and a spread-out candidate set move the realized target by a few hundredths.

\subsection{Further Properties of the Arc-Length Filter}
\label{app:more-theory}

The next three results explain the mechanism measurements of \S\ref{sec:mechanism}. They use the logistic competence model of Appendix~\ref{app:drift}, in which a shared competence gain of $\delta$ per update moves every prompt's logit by $\delta$.

\begin{proposition}[Stale beliefs are biased toward easy prompts]
\label{prop:stale}
Let a belief without drift be centered at a prompt's arc length at its last observation, $\Delta$ updates ago. Then the current arc length exceeds the belief by $\tfrac{\delta\Delta}{4}\sin2\psi+O((\delta\Delta)^2)$, and the current pass rate exceeds the believed one by $\delta\Delta\,p(1-p)+O((\delta\Delta)^2)$. A selector that targets believed pass rate $p^\star$ therefore selects prompts whose true pass rate is $p^\star+\delta\Delta\,p^\star(1-p^\star)$ on average. The bias is largest at $p^\star=1/2$ and grows with the time since the last visit. The drift step of Eq.~\eqref{eq:predict} with $v=\delta/4$ removes it to first order.
\end{proposition}

\emph{Proof.} By Appendix~\ref{app:drift}, $d\psi/dx=\tfrac14\sin2\psi$, so a Taylor expansion gives $\psi(x+\delta\Delta)=\psi(x)+\tfrac{\delta\Delta}{4}\sin2\psi+O((\delta\Delta)^2)$. Multiplying by $dp/d\psi=\sin2\psi$ gives $\Delta p=\tfrac{\delta\Delta}{4}\sin^22\psi=\delta\Delta\,p(1-p)$. The drift step adds exactly $v\sin2\mu$ per update. \qed

For example, with $\delta=0.01$ per update and one epoch between visits ($\Delta\approx66$), a prompt believed at $p=0.5$ is truly at $\approx0.67$. After several epochs the effect compounds, which is the drift of MoPPS and KGPS toward $p\approx0.75$--$0.8$ in Figure~\ref{fig:mechanism}d.

\begin{proposition}[Uniform steady-state uncertainty]
\label{prop:steady}
Suppose a prompt is revisited every $\Delta$ updates and each visit yields an informative group, so the observation variance is $R=1/(4G+2)$ (Eq.~\eqref{eq:obs}). Under the arc-length filter with diffusion $Q$, the posterior variance converges to
\[
P_\infty=\tfrac12\Big(\sqrt{Q^2\Delta^2+4Q\Delta R}-Q\Delta\Big),
\]
which does not depend on the prompt's pass rate. For a logit filter with delta-method noise $R(p)=1/(Gp(1-p))$, the same fixed point is increasing in $R(p)$ and diverges as $p\to\{0,1\}$.
\end{proposition}

\emph{Proof.} One cycle maps $P\mapsto(P+Q\Delta)R/(P+Q\Delta+R)$. A fixed point satisfies $P^2+Q\Delta P-Q\Delta R=0$, whose positive root is $P_\infty$. The map is increasing and concave in $P$ with slope below one at the fixed point, so iterates converge. Finally, $\partial P_\infty/\partial R=Q\Delta/\sqrt{Q^2\Delta^2+4Q\Delta R}>0$. \qed

\begin{proposition}[Waste bound]
\label{prop:waste}
Suppose the beliefs are calibrated, i.e.\ $\psi_q\sim\N(\mu_q,P_q)$ independently, up to truncation to $[0,\hpi]$, and candidates are included with probabilities $\pi_q$ summing to $M$. Then the expected number of zero-variance groups among the candidates is at most $\sum_q\pi_q(1-\yhat_q)=M\big(1-\Yhat_t(\pst)\big)$. For the admissible target of Eq.~\eqref{eq:pace}, the expected number of wasted rollouts per step is at most $GM\big(1-\max_\psi\Yhat_t(\psi)+\varepsilon\big)$.
\end{proposition}

\emph{Proof.} By linearity, the expected number of zero-variance candidates is $\sum_q\pi_q\E[z_G(\psi_q)]$. Proposition~\ref{prop:deadzone} gives $\E[z_G(\psi_q)]\le1-\yhat_q$, and admissibility gives $\Yhat_t(\tilde\psi_t)\ge\max_\psi\Yhat_t(\psi)-\varepsilon$. Each zero-variance candidate wastes $G$ rollouts. When $\Yhat_t$ is unimodal, the rate-limited target lies between two admissible targets and is itself admissible. \qed

Proposition~\ref{prop:waste} is the formal content of ``predicting waste before paying for it''. DS pays $G$ rollouts for every zero-variance group it discards, whereas \method{} bounds the expected number of such groups before generation.

\subsection{Numerical Verification}
\label{app:numerics}

Table~\ref{tab:numerics} reports exact computations (not experiments) that check every analytic statement; the script is \texttt{code/theory\_checks.py}.

\begin{table*}[t]
\centering
\footnotesize
\setlength{\tabcolsep}{5pt}
\begin{tabular}{p{0.50\textwidth} p{0.44\textwidth}}
\toprule
\textbf{Check} & \textbf{Result} \\
\midrule
Prop.~\ref{prop:uniform} vs.\ the finite-group weight of \citet{davis2025objective} & max.\ relative difference $<10^{-14}$ for $G\in\{4,8,16,32\}$\\
$\max_p\omega_G(p)$ vs.\ the bound $\sqrt{1-1/G}$ & $0.928$ vs.\ $0.935$ ($G{=}8$); $0.967$ vs.\ $0.968$ ($G{=}16$)\\
$\min_{p\in[0.1,0.9]}\omega_G(p)$ & $0.51$, $0.70$, $0.85$, $0.94$ for $G=4,8,16,32$\\
Ramp slope $\omega_G/\psi$ as $\psi\to0$ & $2.61$ vs.\ $\sqrt{G-1}=2.65$ ($G{=}8$)\\
Prop.~\ref{prop:deadzone} boundary-layer bound & holds on $[0,\hpi]$ for $G\in\{4,8,16\}$\\
Closed-form $\E[e^{-G\psi^2}]$ under a Gaussian belief & matches Monte Carlo to $5\times10^{-4}$\\
$\mathrm{Var}[\hat\psi_{\mathrm{A}}]$ at $n{=}8$, $p=0.2/0.35/0.5$ & $0.0271/0.0295/0.0297$ vs.\ $1/(4n+2)=0.0294$\\
$\mathrm{Var}[\mathrm{logit}\,\hat p]$ vs.\ delta-method value at $p{=}0.05$, $n{=}8$ & $0.23$ vs.\ $2.63$\\
Posterior $z^2$ after one group, observation model of Eq.~\eqref{eq:obs} & $0.85$ overall; $0.57$--$0.98$ across pass-rate bins\\
Posterior $z^2$ if the Anscombe estimate is also used for zero-variance groups & $1.72$--$1.84$ in the two edge bins (over-confident)\\
Prop.~\ref{prop:kernel} mode $\arcsin(1/\sqrt{2k})$ & exact for $k\in\{1,1.5,2,3,4,8,16\}$\\
Total variation between $H_k$ and its Laplace Gaussian & $0.050$--$0.069$ for $k\in[1,16]$\\
\bottomrule
\end{tabular}
\caption{Exact numerical checks of the analytic results (script \texttt{code/theory\_checks.py}).}
\label{tab:numerics}
\end{table*}

\section{Other Normalizations}
\label{app:drgrpo}

Removing the standard deviation, as in Dr.~GRPO~\citep{liu2025understanding} or leave-one-out baselines~\citep{ahmadian2024back}, changes Proposition~\ref{prop:uniform}. With $\hat A_{q,i}=r_{q,i}-\bar r_q$, the same conditioning argument gives $\E[\hat g_q\mid m]=\frac{m(G-m)}{G^2}\frac{\nabla p}{p(1-p)}$ and, since $\E[m(G-m)]=G(G-1)p(1-p)$, $\E[\hat g_q]=\frac{G-1}{G}\nabla p$. The native coordinate is then $p$ itself: the objective is the mean pass rate, and in arc length the per-prompt weight is $\sin2\psi$ instead of the plateau $\omega_G$. Propositions~\ref{prop:deadzone}--\ref{prop:kernel} are statements about sampling, evidence, and objectives, so they are unaffected. Only the sampler--objective map changes, to $F_w'(\psi)=\frac{G-1}{G}w(\psi)\sin2\psi$. For such losses, \method{} keeps the beliefs, $\yhat$, and pacing, and divides the kernel by $\sin2\psi$, clipped at the grid ends, so that the implicit objective remains the targeted pass@$k$. Table~\ref{tab:compat} shows that this variant retains most of the gain.

\section{Algorithm and Implementation Details}
\label{app:impl}

\paragraph{Belief bookkeeping.}
For each prompt we store $(\mu_q,P_q,t_q)$, where $t_q$ is the index of the last observation. The prediction step of Eq.~\eqref{eq:predict} is applied lazily: when a prompt is scored at step $t$, its belief is advanced by $t-t_q$ steps using the current $(v_t,Q_t)$, which is exact for constant parameters and avoids touching the whole pool each step.

\paragraph{Online drift and diffusion.}
For the set $\mathcal{R}_t$ of revisited prompts observed at step $t$, let $\nu_q=z_q-\mu_q^{\text{pred}}$ be the innovation, $S_q=P_q^{\text{pred}}+R_q$ its predicted variance, and $\Delta_q=t-t_q$ the gap. With rate $\beta=0.1$,
\begin{align*}
v_{t+1}&=v_t+\beta\,\bar\nu_t,\quad
\bar\nu_t=\frac{\sum_{\mathcal{R}_t}\nu_q}{\sum_{\mathcal{R}_t}\Delta_q\,\eta_q},\\
Q_{t+1}&=\max\{10^{-5},(1-\beta)Q_t+\beta[Q_t+\bar e_t]_+\},
\end{align*}
with mobility $\eta_q=\max(\sin2\mu_q,0.05)$ and $\bar e_t=\mathrm{mean}(\nu_q^2-S_q)/\mathrm{mean}(\Delta_q)$. These updates match the mean and the variance of the innovations~\citep{mehra1970identification}, and both estimates stabilize within the warm-up.

\paragraph{Priors.}
Unseen prompts start at $\N(\pi/4,\pi^2/48)$, the moments of the uniform Jeffreys prior on the arc. Once at least 100 prompts have been observed, unseen prompts use the empirical mean of observed $\mu_q$ and the empirical variance of $\mu_q$ plus the mean $P_q$ (an empirical-Bayes prior), which lowers the score of unseen prompts in pools dominated by dead zones.

\paragraph{Inclusion probabilities and the pacing grid.}
For target $\psi$, Gumbel-top-$M$ at temperature $T$ samples without replacement with Plackett--Luce weights $s_q^{1/T}$~\citep{kool2019gumbel}. We approximate its inclusion probabilities by $\pi_q=\min(1,c\,s_q^{1/T})$, with $c$ found by 40 steps of geometric bisection so that $\sum_q\pi_q=M$. The grid $\Psi$ has 41 equally spaced points between $\arcsin(1/\sqrt{2G})$ (pass@$G$) and $\hpi-\arcsin(1/\sqrt{2G})$ (pass$^G$); a group of $G$ rollouts cannot distinguish objectives beyond these. Warm-up lasts $\lceil N/M\rceil$ updates (about 54 for DAPO-Math-17k with $M=320$), and $\Delta=0.005$ rad bounds the per-update target change.

\paragraph{Complexity.}
Each step performs, for every prompt, one belief advance, one kernel overlap, and one dead-zone expectation per grid point, plus 40 bisection passes per grid point: roughly $N\cdot|\Psi|\cdot40\approx2.8\times10^7$ scalar operations for $N{=}17$k, or about $0.05$\,s of vectorized CPU time, which is negligible next to a rollout step of $\approx4$\,min.

\paragraph{Integration.}
In verl, \method{} replaces the training sampler, which yields $M$ prompt indices per step, and adds a hook after reward computation. The hook updates the beliefs from the success counts and filters the batch to the selected informative groups, at the same place where DAPO's filter operates. No change is made to the actor, the critic-free advantage, or the loss.

\begin{algorithm}[t]
\footnotesize
\DontPrintSemicolon
\SetAlgoLined
\caption{Arc-length belief update (one prompt)}
\label{alg:belief}
\KwInput{$(\mu,P,t_q)$, step $t$, successes $m$, $G$, $(v,Q)$}
\For{$\tau=t_q+1,\dots,t$}{$\mu\leftarrow\Pi[\mu+v\sin2\mu]$;\ $P\leftarrow P+Q$\;}
\eIf{$0<m<G$}{$z\leftarrow\arcsin\sqrt{(m+3/8)/(G+3/4)}$;\ $R\leftarrow1/(4G+2)$}
{$z\leftarrow 0$ if $m=0$ else $\hpi$;\ $R\leftarrow1/(2G)$}
$K\leftarrow P/(P+R)$;\ $\mu\leftarrow\Pi[\mu+K(z-\mu)]$;\ $P\leftarrow(1-K)P$;\ $t_q\leftarrow t$\;
\Return{$(\mu,P,t_q)$ and innovation $z-\mu^{\text{pred}}$}
\end{algorithm}

\section{Experimental Setup}
\label{app:setup}

\paragraph{Data and rewards.}
The training pool is DAPO-Math-17k~\citep{yu2025dapo}, deduplicated to unique prompts with integer answers. For the pool-shift study we use the MATH training split~\citep{hendrycks2021math}. The reward is $1$ if the final boxed answer is equivalent to the reference under symbolic and numeric normalization, and $0$ otherwise; there is no format reward. Prompts ask the model to reason step by step and to put the final answer in \texttt{\textbackslash boxed\{\}}, with the chat template of each backbone.

\paragraph{Evaluation.}
AIME24/25 (30 problems each) report mean accuracy over 32 samples, AMC23 (40 problems) over 16 samples, and MATH500~\citep{lightman2024verify}, Minerva Math~\citep{lewkowycz2022minerva}, and OlympiadBench~\citep{he2024olympiadbench} over 4 samples, all at temperature $0.6$, top-$p$ $0.95$, and the training response limit. pass@$k$ uses the unbiased estimator of \citet{chen2021codex} with 32 samples per problem. We evaluate the final checkpoint and report the mean over three seeds; standard deviations are in Table~\ref{tab:seeds}.

\paragraph{Hyperparameters.}
Table~\ref{tab:hparams} lists the shared configuration. \method{}'s hyperparameters were chosen on a held-out set of 500 DAPO prompts with Qwen2.5-Math-1.5B and then fixed for all backbones and pools.

\begin{table}[t]
\centering
\footnotesize
\setlength{\tabcolsep}{3pt}
\begin{tabular}{l l}
\toprule
\textbf{Hyperparameter} & \textbf{Value} \\
\midrule
Framework & verl + vLLM, bf16 \\
Policy updates & 300 \\
Prompts per update $B$ & 256 \\
Group size $G$ & 8 \\
Mini-batch (prompts) & 64 (4 steps per update) \\
Optimizer & AdamW, lr $10^{-6}$, wd $0.01$ \\
Clip range & $0.2$ \\
KL / entropy coefficient & $0$ / $0$ \\
Rollout temperature, top-$p$ & $1.0$, $1.0$ \\
Max prompt length & 1024 \\
Max response length & 3072 (Q2.5-Math); 8192 (Q3) \\
Hardware & 8$\times$H100 80GB \\
\midrule
\multicolumn{2}{l}{\textit{\method{}}}\\
Candidate margin $\rho$ & $0.25$ ($M=320$) \\
Yield slack $\varepsilon$ & $0.03$ \\
Gumbel temperature $T$ & $0.3$ \\
Target grid size $|\Psi|$ & 41 \\
Target rate limit $\Delta$ & $0.005$ rad per update \\
Warm-up & $\lceil N/M\rceil\approx54$ updates \\
Estimator rate $\beta$ & $0.1$ \\
Initial belief & Jeffreys, then empirical Bayes \\
\bottomrule
\end{tabular}
\caption{Training and \method{} hyperparameters.}
\label{tab:hparams}
\end{table}

\paragraph{Baselines.}
All baselines share the GRPO loss and Table~\ref{tab:hparams}. \emph{DS}: candidate batches of $B$ prompts are generated until $B$ informative groups are collected (at most 8 rounds). \emph{GRESO}: probabilistic pre-rollout skipping of prompts with recent zero-variance history, followed by DS-style filling, with the published skip schedule. \emph{MoPPS}: Beta$(1,1)$ priors with the published decay, Thompson sampling, and prompts whose sampled pass rates are closest to $0.5$. \emph{DPS}: the published hidden-Markov configuration. \emph{KGPS}: logit-space Kalman filter with a uniform warm-up epoch, process noise coupled to the parameter change, and $\E[p(1-p)]$ by 5-point Gauss--Hermite quadrature. \emph{CurES}: softmax over $\sqrt{p(1-p)}$ of posterior-mean pass rates with the published temperature. \emph{GCS}: Gaussian weights ($\mu{=}0.5$, $\sigma{=}0.35$) over EMA pass rates ($\alpha{=}0.9$), as in FG-ExPO but without its KL schedule, to isolate sampling. \emph{AdaRFT}: static difficulty from 8 base-model samples per prompt, with the target difficulty moved by the batch reward toward $0.5$. Predictive baselines sample exactly $B$ prompts per update.

\paragraph{Compute.}
One 7B run takes 19.8\,h (GRPO) to 41.6\,h (DS) on 8$\times$H100 (Table~\ref{tab:cost}); Table~\ref{tab:wall} breaks the time into generation, reward computation, policy update, and selection. The full study, covering three backbones, eleven methods, three seeds, ablations, and appendix settings, used roughly $3\times10^4$ H100-hours.

\begin{table}[h]
\centering
\footnotesize
\setlength{\tabcolsep}{3pt}
\input{tables/wall.tex}
\caption{Wall-clock breakdown (hours) on Qwen2.5-Math-7B, 8$\times$H100. Selection is \method{}'s belief update, scoring, and pacing, which is negligible; its extra time is generating the candidate margin.}
\label{tab:wall}
\end{table}

\paragraph{Reproducibility checklist.}
(i) The belief update, scoring, sampling, and pacing are fully specified by Eqs.~\eqref{eq:predict}--\eqref{eq:pace}, Algorithm~\ref{alg:belief}, and the paragraphs above. (ii) All hyperparameters are listed in Table~\ref{tab:hparams}, and none is tuned per backbone. (iii) Each baseline uses its published configuration with the shared GRPO settings. (iv) Evaluation uses fixed sampling parameters and three training seeds. (v) The analytic figures and Table~\ref{tab:numerics} are produced by the scripts in \texttt{code/}. (vi) The only randomness in \method{} beyond the policy's sampling is the Gumbel noise of candidate selection, which is seeded.

\section{Additional Results}
\label{app:extra}

\paragraph{Seed variability.}
Table~\ref{tab:seeds} reports the six-benchmark average with standard deviations over three seeds. \method{}'s gain over DS exceeds three standard deviations on every backbone, and its variance is the lowest among the compared methods. We attribute this to steadier batches: the informative share of its update batch varies by less than two points across steps.

\begin{table}[h]
\centering
\footnotesize
\setlength{\tabcolsep}{3pt}
\input{tables/seeds.tex}
\caption{Six-benchmark average, mean $\pm$ s.d.\ over three seeds.}
\label{tab:seeds}
\end{table}

\paragraph{Per-benchmark dynamics.}
Figure~\ref{fig:bench} shows training curves per benchmark on Qwen2.5-Math-7B. On MATH500 the methods converge together; the separation arises on AIME and AMC and opens after the warm-up, when the paced target starts to move.

\begin{figure*}[t]
\centering
\includegraphics[width=0.92\textwidth]{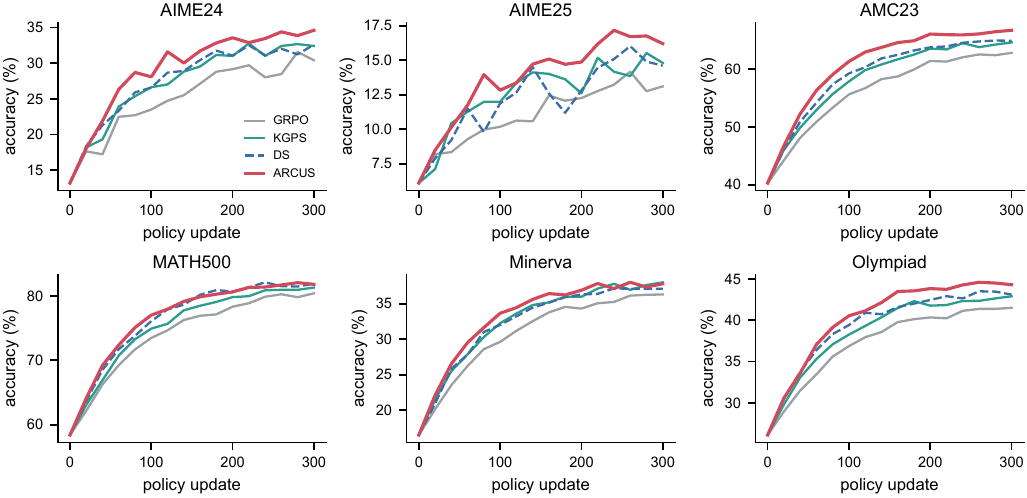}
\caption{Per-benchmark accuracy over training on Qwen2.5-Math-7B.}
\label{fig:bench}
\end{figure*}

\paragraph{Extended ablations.}
Table~\ref{tab:ablation-appx} complements Table~\ref{tab:ablation}. Using the Anscombe observation for zero-variance groups instead of the boundary-layer likelihood costs $0.57$. This is the calibration failure of Figure~\ref{fig:geometry}c: the belief of a dead prompt stays near $p\approx0.04$ instead of collapsing into the dead zone. The exact pass@$k$ kernel, evaluated by quadrature, is indistinguishable from the Gaussian, which confirms that the Gaussian is a convenience, not an approximation that costs accuracy. A larger candidate margin ($\rho=0.5$) buys $+0.12$ for 20\% more rollouts, so $\rho=0.25$ is near the knee of the cost curve (Figure~\ref{fig:results}d).

\begin{table}[h]
\centering
\footnotesize
\setlength{\tabcolsep}{2.4pt}
\resizebox{\columnwidth}{!}{\input{tables/ablation_appx.tex}}
\caption{Additional ablations on Qwen2.5-Math-7B.}
\label{tab:ablation-appx}
\end{table}

\paragraph{Objective family.}
Table~\ref{tab:sweep} reports the fixed-target sweep behind Figure~\ref{fig:results}c. Moving the target from pass$^4$ to pass@4 raises pass@16 monotonically from $64.9$ to $69.8$, while pass@1 peaks between pass@1 and pass@2; pass@8 fails because its target sits inside the dead zone (47\% yield). Paced \method{} is better on pass@1 than every fixed target and nearly the best on pass@16. The best trade-off is thus not a fixed objective but a moving one.

\begin{table}[h]
\centering
\footnotesize
\setlength{\tabcolsep}{3pt}
\resizebox{\columnwidth}{!}{\input{tables/objective_sweep.tex}}
\caption{Fixed targets along the objective family vs.\ paced \method{} (Qwen2.5-Math-7B; six-benchmark averages).}
\label{tab:sweep}
\end{table}

\paragraph{Group size.}
Table~\ref{tab:group} varies $G$ at a matched number of updates. DS's overhead shrinks with $G$, because fewer groups are zero-variance, but \method{}'s advantage persists. Its final target hardens with $G$ as predicted by the frontier analysis of Appendix~\ref{app:frontier}.

\begin{table}[h]
\centering
\footnotesize
\setlength{\tabcolsep}{3pt}
\input{tables/group_size.tex}
\caption{Effect of group size on Qwen2.5-Math-7B. Frontier: closed-form sharp-belief target for $\varepsilon=0.03$.}
\label{tab:group}
\end{table}

\paragraph{Sensitivity and coverage.}
Figure~\ref{fig:sens} sweeps the yield slack $\varepsilon$ and the Gumbel temperature $T$. Both have broad optima, and every tested value beats DS. Too small an $\varepsilon$ freezes the target near pass@1, and too large an $\varepsilon$ admits wasteful targets; $T\to0$ over-exploits stale beliefs, while $T=1$ dilutes the candidate set. On AIME24, \method{} improves pass@$k$ at every $k\le32$ and widens the gap at large $k$ ($66.4$ vs.\ $60.2$ for GRPO at $k=32$), whereas GRPO's pass@32 barely exceeds the base model's ($58.1$), consistent with reports that RLVR tends to narrow coverage~\citep{yue2025does}.

\begin{figure*}[t]
\centering
\includegraphics[width=\textwidth]{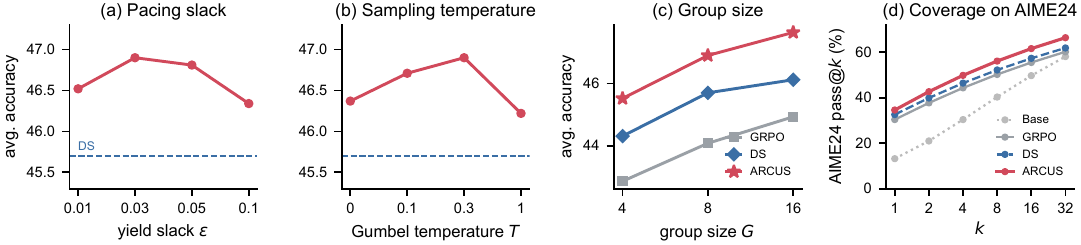}
\caption{Sensitivity on Qwen2.5-Math-7B: (a) yield slack, (b) Gumbel temperature, (c) group size, and (d) pass@$k$ on AIME24.}
\label{fig:sens}
\end{figure*}

\paragraph{Additional backbones and pools.}
Table~\ref{tab:extra} covers a long-chain-of-thought model (DeepSeek-R1-Distill-Qwen-1.5B~\citep{deepseekai2025r1} with an 8k response limit), a non-Qwen family (Llama-3.2-3B-Instruct~\citep{grattafiori2024llama3}), and a different pool (MATH-train for Qwen2.5-Math-7B). The Llama backbone rules out a Qwen-specific artifact~\citep{shao2025spurious}. Its pool is dominated by all-fail groups, so DS costs $\times3.5$ rollouts, and \method{}'s dead-zone factor matters most there.

\begin{table}[h]
\centering
\footnotesize
\setlength{\tabcolsep}{2.2pt}
\resizebox{\columnwidth}{!}{\input{tables/extra.tex}}
\caption{Additional backbones and training pools (six-benchmark average; DS roll.: DS rollouts relative to GRPO; \method{} uses $\times1.25$).}
\label{tab:extra}
\end{table}

\begin{table*}[t]
\centering
\footnotesize
\setlength{\tabcolsep}{6pt}
\input{tables/extra_full.tex}
\caption{Per-benchmark results for the additional backbones and the MATH-train pool (accuracy, \%; same protocol as Table~\ref{tab:main}).}
\label{tab:extra-full}
\end{table*}

Table~\ref{tab:extra-full} gives the per-benchmark breakdown. As on the main backbones, the gains concentrate on AIME and AMC. For the long-CoT model, \method{} is below DS only on MATH500 ($86.3$ vs.\ $86.4$).

\paragraph{Batch size, learning rate, and pool size.}
Table~\ref{tab:robust} varies the prompts per update, the learning rate, and the size of the training pool (random subsets of DAPO-Math-17k). \method{}'s margin over DS stays between $+0.9$ and $+1.3$ in every setting. On smaller pools, prompts are revisited more often, so beliefs are sharper, but the pool also holds fewer informative prompts near the frontier. The two effects roughly cancel, and the margin shrinks only slightly, to $+0.9$ on the 4k pool.

\begin{table}[h]
\centering
\footnotesize
\setlength{\tabcolsep}{4pt}
\input{tables/robust.tex}
\caption{Robustness to batch size $B$, learning rate, and pool size on Qwen2.5-Math-7B (six-benchmark average).}
\label{tab:robust}
\end{table}

\paragraph{Cold start.}
During warm-up every unseen prompt shares the same prior, so the first pass over the pool is nearly uniform. Table~\ref{tab:coldstart} compares ways of initializing the beliefs. The empirical-Bayes prior of Appendix~\ref{app:impl} is free and already helps. A reference-model prior, i.e.\ a difficulty estimate from a smaller model's pass rates in the spirit of \citet{sha2026thinkprior}, and an offline probe with two base-model samples per prompt (3.4\% extra rollouts) raise the warm-up yield from 51\% to 62--66\% and add $0.08$--$0.13$ to the final average. \method{} is therefore compatible with cold-start priors, but it does not depend on them.

\begin{table}[h]
\centering
\footnotesize
\setlength{\tabcolsep}{3pt}
\resizebox{\columnwidth}{!}{\input{tables/coldstart.tex}}
\caption{Initial beliefs on Qwen2.5-Math-7B. Warm-up yield: informative share of generated groups during the first $\lceil N/M\rceil$ updates.}
\label{tab:coldstart}
\end{table}

\paragraph{Other policy losses.}
Table~\ref{tab:compat} swaps the GRPO loss for Dr.~GRPO, the DAPO token-level loss with clip-higher, and RLOO. \method{} improves on DS for every loss; for Dr.~GRPO and RLOO we use the adapted kernel of Appendix~\ref{app:drgrpo}. The gain is smallest for losses without standard-deviation normalization, whose native coordinate is $p$ rather than $\psi$.

\begin{table}[h]
\centering
\footnotesize
\setlength{\tabcolsep}{3pt}
\input{tables/compat.tex}
\caption{Compatibility with other group-based losses (Qwen2.5-Math-7B, six-benchmark average).}
\label{tab:compat}
\end{table}

\paragraph{Training dynamics.}
Figure~\ref{fig:dynamics} tracks response length, token entropy, and filter consistency. Training on harder prompts lengthens responses, and \method{} keeps entropy highest ($0.27$ nats at the end vs.\ $0.17$ for GRPO), in line with its better pass@$k$. The normalized innovation squared of the arc-length filter settles near one after warm-up: the filter's predicted uncertainty matches its errors.

\begin{figure*}[t]
\centering
\includegraphics[width=0.92\textwidth]{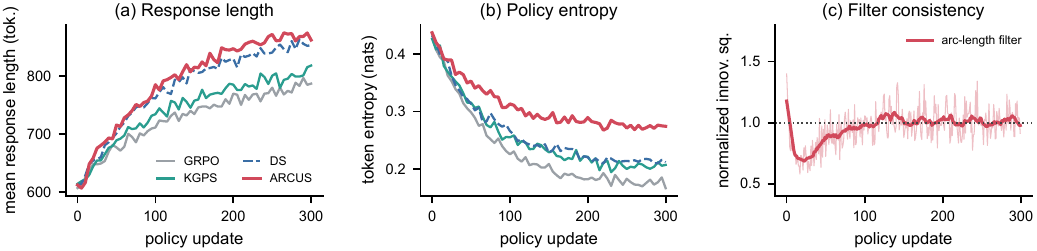}
\caption{Training dynamics on Qwen2.5-Math-7B: (a) mean response length, (b) token entropy, and (c) normalized innovation squared of \method{}'s filter (1 = consistent).}
\label{fig:dynamics}
\end{figure*}

\section{Discussion}
\label{app:discussion}

\paragraph{Why one coordinate explains four facts.}
The four propositions share one cause: the Fisher information of a Bernoulli trial is constant in $\psi$. This makes GRPO's update flat, because standard-deviation normalization divides by $\sqrt{p(1-p)}$, the square root of the Fisher information in $p$. It makes evidence homoscedastic, because each rollout carries information $4$ about $\psi$. And it makes dead zones and objective kernels Gaussian, because a run of $n$ identical outcomes, or a success among $k$ attempts, is a product of $n$ or $k$ Bernoulli likelihoods whose log-curvature in $\psi$ is $2n$ or $4k$. None of this is specific to mathematics; it holds for any binary verifier.

\paragraph{When the geometry is not native.}
Without standard-deviation normalization (Appendix~\ref{app:drgrpo}) the optimizer's native coordinate becomes $p$, although sampling and evidence still live on the arc. With continuous rewards, a variance-stabilizing map depends on the reward distribution; for bounded rewards, binarizing at a threshold or using a Beta-mean model is a natural extension. With very small groups ($G\le3$), the two dead zones cover most of the arc, and any selector struggles.

\paragraph{Sampler-side versus advantage-side objectives.}
\citet{davis2025objective} show that advantages choose the objective; Corollary~\ref{cor:objective} shows that samplers do too. The two routes differ in cost. Reshaping advantages to target pass@$k$ still pays for every zero-variance group, while reshaping the sampler avoids paying for them. The routes also compose: a pass@$k$ advantage combined with a pass@$k$-matched sampler would target the same objective with less waste, which we leave to future work.

\paragraph{The easy-to-hard pattern.}
Classical curricula schedule difficulty; \method{} schedules an objective and lets waste decide how fast it can move. The emergent trajectory (Figure~\ref{fig:mechanism}a) starts at pass@1, because early beliefs cannot certify harder targets, and hardens as beliefs sharpen. This is a concrete, measurable version of training at the frontier of learnability.

\paragraph{Failure modes.}
(i) Pools with very few informative prompts: the pacing retreats to the most informative target, and the candidate margin may not fill the batch; \method{} then degrades gracefully to a well-calibrated predictive selector. (ii) Strongly heterogeneous transfer: a single global drift underestimates the movement of fast-improving topics; per-cluster drift is a straightforward extension. (iii) Cost heterogeneity: all-fail responses are longer, so a token-cost-aware score $s_q/\E[\text{tokens}]$ could save more compute than rollout counts suggest. (iv) Adversarially duplicated prompts share evidence that the filter treats as independent.

\paragraph{Practical recommendations.}
The defaults ($\rho{=}0.25$, $\varepsilon{=}0.03$, $T{=}0.3$) transferred across all backbones and pools we tried. If the rollout budget must equal GRPO's, use $\rho{=}0$, which keeps most of the gain (Table~\ref{tab:main}). If the pool is dominated by all-fail prompts, as for weaker models, a smaller $\varepsilon$ keeps the target closer to pass@1 until beliefs sharpen. If pass@$k$ coverage matters more than pass@1, a larger $\varepsilon$ or a fixed pass@$k$ target is the knob to turn (Table~\ref{tab:sweep}). Monitoring the predicted and realized yield together is a cheap diagnostic: a persistent gap signals miscalibrated beliefs, typically a drift estimate that has not yet converged.

\paragraph{Broader impact.}
\method{} reduces the compute and energy of RLVR by about half relative to dynamic sampling at equal or better accuracy. Cheaper post-training lowers barriers for academic groups and also for misuse. The method adds no new data or capabilities beyond those of the underlying training pipeline.

%% file: tables/wall.tex
\begin{tabular}{l cccc c}
\toprule
\textbf{Method} & \textbf{Gen.} & \textbf{Reward} & \textbf{Update} & \textbf{Select} & \textbf{Total} \\
\midrule
GRPO & 13.9 & 0.4 & 5.5 & $<$0.01 & 19.8 \\
KGPS & 14.0 & 0.4 & 5.6 & 0.01 & 20.0 \\
DS & 34.9 & 1.2 & 5.5 & $<$0.01 & 41.6 \\
\rowcolor{oursbg}\method{}$_{\rho=0}$ & 13.9 & 0.4 & 5.7 & 0.01 & 20.0 \\
\rowcolor{oursbg}\method{} & 18.0 & 0.5 & 5.6 & 0.01 & 24.1 \\
\bottomrule
\end{tabular}

%% file: tables/seeds.tex
\begin{tabular}{l ccc}
\toprule
\textbf{Method} & \textbf{Q2.5-M-1.5B} & \textbf{Q2.5-M-7B} & \textbf{Q3-4B-Base} \\
\midrule
GRPO & $34.82\pm0.41$ & $44.08\pm0.37$ & $46.27\pm0.39$ \\
MoPPS & $35.73\pm0.38$ & $45.03\pm0.35$ & $47.27\pm0.36$ \\
KGPS & $36.32\pm0.36$ & $45.67\pm0.34$ & $47.82\pm0.33$ \\
DS & $36.50\pm0.33$ & $45.70\pm0.30$ & $48.00\pm0.31$ \\
\rowcolor{oursbg}\method{}$_{\rho=0}$ & $36.97\pm0.31$ & $46.25\pm0.30$ & $48.50\pm0.29$ \\
\rowcolor{oursbg}\method{} & $37.65\pm0.29$ & $46.90\pm0.27$ & $49.20\pm0.28$ \\
\bottomrule
\end{tabular}

%% file: tables/ablation_appx.tex
\begin{tabular}{l c c c}
\toprule
\textbf{Configuration} & \textbf{Avg.} & $\bm{\Delta}$ & \textbf{Inf.\,(\%)} \\
\midrule
\rowcolor{oursbg}\method{} (full) & \textbf{46.90} & -- & 97 \\
\midrule
\multicolumn{4}{l}{\textit{Belief dynamics}}\\
\quad No mobility-weighted drift & 46.52 & $-0.38$ & 95 \\
\quad Jeffreys prior only & 46.71 & $-0.19$ & 96 \\
\quad Anscombe obs.\ for zero-var.\ groups & 46.33 & $-0.57$ & 93 \\
\midrule
\multicolumn{4}{l}{\textit{Objective kernel}}\\
\quad Exact pass@$k$ kernel (quadrature) & 46.87 & $-0.03$ & 97 \\
\quad Sampling $\propto$ score ($T{=}1$) & 46.22 & $-0.68$ & 91 \\
\quad $\omega_G$ in place of $\hat y$ & 46.71 & $-0.19$ & 93 \\
\midrule
\multicolumn{4}{l}{\textit{Sampling and selection}}\\
\quad $\rho{=}0.10$ & 46.58 & $-0.32$ & 90 \\
\quad $\rho{=}0.50$ & 47.02 & $+0.12$ & 99 \\
\bottomrule
\end{tabular}

%% file: tables/objective_sweep.tex
\begin{tabular}{l c c c c}
\toprule
\textbf{Target} & $p^\star$ & \textbf{pass@1} & \textbf{pass@16} & \textbf{Yield} \\
\midrule
pass$^4$ & 0.875 & 44.61 & 64.9 & 72\% \\
pass$^2$ & 0.75 & 45.43 & 66.1 & 85\% \\
pass@1 & 0.5 & 46.28 & 67.6 & 88\% \\
pass@1.5 & 0.333 & 46.47 & 68.6 & 84\% \\
pass@2 & 0.25 & 46.31 & 69.3 & 78\% \\
pass@4 & 0.125 & 45.62 & 69.8 & 63\% \\
pass@8 & 0.0625 & 44.35 & 68.9 & 47\% \\
\midrule
\rowcolor{oursbg}\method{} & paced & \textbf{46.90} & 69.6 & 79\% \\
GRPO & -- & 44.08 & 66.3 & -- \\
DS & -- & 45.70 & 67.4 & -- \\
KGPS & -- & 45.67 & 67.1 & -- \\
\bottomrule
\end{tabular}

%% file: tables/group_size.tex
\begin{tabular}{c ccc c c}
\toprule
$G$ & \textbf{GRPO} & \textbf{DS} (roll.) & \textbf{\method{}} & \textbf{final} $p^\star_T$ & \textbf{frontier} \\
\midrule
4 & 42.86 & 44.31 ($\times$3.37) & \textbf{45.52} & 0.43 & 0.40 \\
8 & 44.08 & 45.70 ($\times$2.41) & \textbf{46.90} & 0.29 & 0.34 \\
16 & 44.93 & 46.12 ($\times$2.02) & \textbf{47.64} & 0.18 & 0.20 \\
\bottomrule
\end{tabular}

%% file: tables/extra.tex
\begin{tabular}{l cccc c}
\toprule
\textbf{Setting} & \textbf{GRPO} & \textbf{KGPS} & \textbf{DS} & \textbf{\method{}} & \textbf{DS roll.} \\
\midrule
R1-Distill-1.5B (8k) & 47.86 & 48.95 & 49.12 & \textbf{50.21} & $\times$2.12 \\
Llama-3.2-3B-Inst. & 21.37 & 22.31 & 22.48 & \textbf{23.26} & $\times$3.46 \\
Q2.5-Math-7B, MATH pool & 43.15 & 44.30 & 44.46 & \textbf{45.38} & $\times$2.78 \\
\bottomrule
\end{tabular}

%% file: tables/extra_full.tex
\begin{tabular}{l cccccc c}
\toprule
\textbf{Method} & \textbf{AIME24} & \textbf{AIME25} & \textbf{AMC23} & \textbf{MATH500} & \textbf{Minerva} & \textbf{Olymp.} & \textbf{Avg.} \\
\midrule
\multicolumn{8}{c}{\textit{R1-Distill-1.5B (8k)}}\\
\midrule
GRPO & 33.4 & 25.1 & 71.0 & 85.2 & 30.4 & 42.1 & 47.87 \\
KGPS & 35.2 & 26.0 & 72.4 & 85.9 & 30.9 & 43.3 & 48.95 \\
DS & 35.0 & 26.4 & 72.9 & 86.4 & 30.8 & 43.2 & 49.12 \\
\rowcolor{oursbg}\method{} & 37.4 & 27.6 & 74.1 & 86.3 & 31.5 & 44.4 & 50.22 \\
\midrule
\multicolumn{8}{c}{\textit{Llama-3.2-3B-Inst.}}\\
\midrule
GRPO & 5.9 & 1.8 & 24.1 & 50.6 & 17.2 & 28.6 & 21.37 \\
KGPS & 6.7 & 2.3 & 26.2 & 51.9 & 18.1 & 28.7 & 22.32 \\
DS & 6.5 & 2.4 & 26.5 & 52.4 & 17.8 & 29.3 & 22.48 \\
\rowcolor{oursbg}\method{} & 7.8 & 2.9 & 27.4 & 52.6 & 18.4 & 30.5 & 23.27 \\
\midrule
\multicolumn{8}{c}{\textit{Q2.5-Math-7B, MATH pool}}\\
\midrule
GRPO & 28.6 & 12.2 & 61.9 & 81.3 & 35.7 & 39.2 & 43.15 \\
KGPS & 30.4 & 13.1 & 63.6 & 81.9 & 36.8 & 40.0 & 44.30 \\
DS & 30.7 & 13.3 & 63.9 & 82.4 & 36.2 & 40.3 & 44.47 \\
\rowcolor{oursbg}\method{} & 32.3 & 14.4 & 65.2 & 82.5 & 36.9 & 41.0 & 45.38 \\
\bottomrule
\end{tabular}

%% file: tables/robust.tex
\begin{tabular}{l ccc}
\toprule
\textbf{Setting} & \textbf{GRPO} & \textbf{DS} & \textbf{\method{}} \\
\midrule
$B{=}128$ & 43.52 & 45.11 & \textbf{46.33} \\
$B{=}256$ (default) & 44.08 & 45.70 & \textbf{46.90} \\
$B{=}512$ & 44.61 & 46.02 & \textbf{47.21} \\
lr $5{\times}10^{-7}$ & 43.12 & 44.71 & \textbf{45.98} \\
lr $2{\times}10^{-6}$ & 43.89 & 45.38 & \textbf{46.64} \\
\midrule
pool 4k & 43.61 & 44.95 & \textbf{45.87} \\
pool 8k & 43.93 & 45.38 & \textbf{46.51} \\
pool 17k & 44.08 & 45.70 & \textbf{46.90} \\
\bottomrule
\end{tabular}

%% file: tables/coldstart.tex
\begin{tabular}{l ccc}
\toprule
\textbf{Initial belief} & \textbf{Avg.} & \textbf{Warm-up yield} & \textbf{Extra roll.} \\
\midrule
Jeffreys moments only & 46.71 & 47\% & 0.0\% \\
\rowcolor{oursbg}Empirical Bayes (default) & 46.90 & 51\% & 0.0\% \\
Reference-model prior & 46.98 & 62\% & 0.0\% \\
Offline probe, 2 samples/prompt & 47.03 & 66\% & 3.4\% \\
\bottomrule
\end{tabular}

%% file: tables/compat.tex
\begin{tabular}{l ccc c}
\toprule
\textbf{Policy loss} & \textbf{Uniform} & \textbf{DS} & \textbf{\method{}} & $\bm{\Delta}$ vs.\ DS \\
\midrule
GRPO & 44.08 & 45.70 & \textbf{46.90} & +1.20 \\
Dr.~GRPO & 43.92 & 45.41 & \textbf{46.47} & +1.06 \\
DAPO loss & 44.61 & 46.03 & \textbf{47.18} & +1.15 \\
RLOO & 43.54 & 45.02 & \textbf{46.11} & +1.09 \\
\bottomrule
\end{tabular}